\documentclass{article}

\usepackage{arxiv}

\usepackage[utf8]{inputenc} % allow utf-8 input
\usepackage[T1]{fontenc}    % use 8-bit T1 fonts
\usepackage{xurl}           % simple URL typesetting
\usepackage{booktabs}       % professional-quality tables
\usepackage{amsmath}
\usepackage{amssymb}
\usepackage{amsfonts}       % blackboard math symbols
\usepackage{mathtools}
\usepackage{amsthm}
\usepackage{bm}
\usepackage{bbm}
\usepackage{nicefrac}       % compact symbols for 1/2, etc.
\usepackage{microtype}      % microtypography
\usepackage{graphicx}
\usepackage{algorithm}
\usepackage[noend]{algorithmic}
\usepackage{tikz}
\usetikzlibrary{arrows.meta}
\usepackage{placeins}
\usepackage{multirow}
\usepackage[numbers,sort&compress]{natbib}
\usepackage{hyperref}       % hyperlinks
\usepackage{doi}

\newtheoremstyle{lncsbold}{\topsep}{\topsep}{\itshape}{}{\bfseries}{.}{ }%
  {\thmname{#1}\thmnumber{ #2}\thmnote{ (#3)}}
\newtheoremstyle{lncsital}{\topsep}{\topsep}{\normalfont}{}{\itshape}{.}{ }%
  {\thmname{#1}\thmnumber{ #2}\thmnote{ (#3)}}
\theoremstyle{lncsbold}
\newtheorem{theorem}{Theorem}
\newtheorem{lemma}{Lemma}
\newtheorem{definition}{Definition}

\theoremstyle{lncsital}
\newtheorem{remark}{Remark}

\newcommand{\TopL}{\mathrm{Top}_L}
\newcommand{\Reg}{\mathrm{Reg}}
\newcommand{\KL}{d}

\newcommand{\E}{\mathbb{E}}

\DeclareMathOperator*{\argmax}{arg\,max}
\allowdisplaybreaks[1]

\title{Ranking Bandits for Carousel Interfaces with Observable Browsing Depth}

\author{Takuma Yasuda \hspace{1cm} Atsuyoshi Nakamura \hspace{1cm} Dai Hai Nguyen\\
	Graduate School of Information Science 
	and Technology \\
	Hokkaido University \\
	Sapporo, Japan \\
	\texttt{\{yasuda.takuma.q4@elms, atsu@ist, hai@ist\}.hokudai.ac.jp}\\
}

\renewcommand{\shorttitle}{Ranking Bandits with Observable Browsing Depth}

\hypersetup{
pdftitle={Ranking Bandits for Carousel Interfaces with Observable Browsing Depth},
pdfauthor={Takuma Yasuda, Atsuyoshi Nakamura},
pdfkeywords={Bandit algorithms, Learning to rank, Recommender systems, Carousel interfaces, Online learning},
}

\begin{document}
\maketitle

\begin{abstract}
Carousel interfaces allow a recommender system to directly observe how far
a user has browsed. This signal distinguishes displayed but unclicked items
from items that were never displayed, whereas conventional ranking-bandit
models, including cascade and position-based models, generally treat
examination as latent. We formulate a ranking-bandit problem in which a
learner presents a list of \(L\) items, observes the user's maximum browsing
depth, and receives click feedback only for positions up to that depth. The
objective is to maximize the expected number of clicks under an unknown
item-attractiveness vector and a browsing-depth distribution. We propose
three algorithms based on UCB, Thompson Sampling, and DMED,
all of which update item statistics only from observed
exposures. We derive an instance-dependent logarithmic upper bound for our
UCB-based algorithm and an asymptotic upper bound for our DMED-based
algorithm that coincides with the lower bound as its parameter \(\alpha\downarrow0\),
establishing asymptotic optimality in this limit. 
Simulations in synthetic shallow- and
deep-browsing environments, together with experiments parameterized from
RecGaze interaction logs, show that OD-TS attains final mean regret similar
to PBM-TS, while the proposed methods achieve lower final mean regret than
PBM-UCB.

\end{abstract}

% keywords can be removed

\section{Introduction}

Online learning from user feedback, such as clicks and dwell time, is central
to Web recommendation. Search engines, news services, and e-commerce sites
present multiple ranked items, and ranking bandits provide a principled way
to balance exploration of uncertain items with exploitation of items that
have performed well.

Much of the ranking-bandit literature builds on click models. The cascade
model assumes sequential examination from the top of a list, while the
position-based model (PBM) represents click probability as the product of
position-dependent examination and item attractiveness
\cite{Kveton2015Cascading,Lagree2016PBM}. Examination is usually latent:
after a non-click, the learner cannot tell whether the user examined and
rejected an item or never saw it. PBM methods must additionally separate
position bias from attractiveness, which can reduce statistical efficiency.

Carousel and scrollable-list interfaces provide an additional signal. Client
events can reveal the maximum position actually displayed, which we call the
\emph{browsing depth}. For positions at or above this depth, both clicks and
confirmed non-clicks are available; lower positions remain unobserved. Direct
depth observations therefore prevent unexposed items from being mixed with
negative examples.

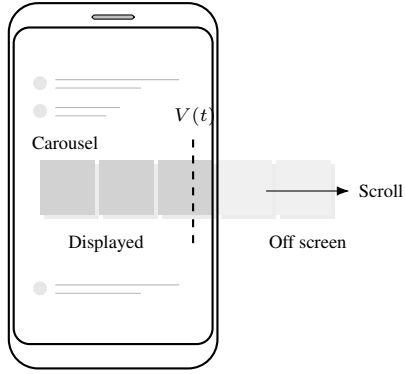
\begin{figure}[t]
  \centering
  \begin{tikzpicture}[x=0.92cm,y=0.92cm,font=\scriptsize,>=Latex]
    \fill[white] (0.40,0.45) rectangle (3.10,5.10);
    \draw[rounded corners=1pt,fill=gray!50] (1.45,5.12) rectangle (2.05,5.18);
    \foreach \y in {4.20,1.25} {
      \fill[gray!20] (0.70,\y) circle (0.10);
      \draw[gray!60] (0.92,\y+0.05) -- (2.70,\y+0.05);
      \draw[gray!55] (0.92,\y-0.07) -- (2.15,\y-0.07);
    }
    \foreach \y in {3.80} {
      \fill[gray!20] (0.70,\y) circle (0.10);
      \draw[gray!60] (0.92,\y+0.05) -- (1.85,\y+0.05);
      \draw[gray!55] (0.92,\y-0.07) -- (1.65,\y-0.07);
    }
    \node[anchor=west] at (0.45,3.35) {Carousel};
    \foreach \x/\tone in {0.70/35,1.55/35,2.40/35,3.25/12,4.10/12} {
      \fill[gray!15] (\x+0.06,2.25) rectangle ++(0.78,0.78);
      \fill[gray!\tone] (\x,2.31) rectangle ++(0.78,0.78);
    }
    \draw[dashed,thick] (2.90,1.90) -- (2.90,3.50);
    \node[above] at (2.93,3.45) {\(V(t)\)};
    \node[below] at (1.65,2.15) {Displayed};
    \node[below] at (4.55,2.15) {Off screen};
    \draw[->] (3.95,2.65) -- (5.15,2.65) node[right]{Scroll};
    \draw[rounded corners=11pt,thick] (0.25,0.10) rectangle (3.25,5.35);
    \draw[rounded corners=6pt,thick] (0.32,0.45) rectangle (3.18,5.00);
  \end{tikzpicture}
  \caption{A carousel interface in which positions up to browsing depth
  \(V(t)\) are observed.}
  \label{fig:carousel}
\end{figure}

We formulate a ranking bandit for this observable-depth setting. In every
round, the learner presents \(L\) distinct items and observes clicks only up
to the user's maximum displayed position. Expected reward is the sum, over
positions, of the exposure probability times the attractiveness of the item
placed there. To capture carousel interactions in which multiple items may
be clicked and to isolate the value of depth observation, we use a simple
multi-click process in which every exposed item is clicked independently
according to its attractiveness. We further assume that browsing depth is
independent of item display order.

Our contributions are as follows:
\begin{itemize}
  \item We formulate a ranking-bandit problem for carousel interfaces with
  directly observable browsing depth.
  \item We introduce OD-UCB, OD-TS, and OD-DMED-L algorithms,
  each of which updates only from confirmed exposures.
  \item We establish an instance-dependent upper bound for OD-UCB. The
  asymptotic OD-DMED-L upper bound coincides with the lower bound as its parameter
  \(\alpha\downarrow0\), establishing optimality in this limit.
  \item Synthetic and RecGaze-driven experiments show that OD-TS attains
  final mean regret similar to PBM-TS, while the proposed methods achieve
  lower final mean regret than PBM-UCB.
\end{itemize}

\subsection*{Related Work}
\label{sec:related}
\iffalse
This subsection reviews three research areas closely related to our work:
basic bandit algorithms, ranking bandits based on click models such as the
cascade model and PBM, and behavioral modeling for carousel and other
multi-list interfaces. We then clarify the position of the present work.
\fi
%\subsubsection{Bandit Algorithms:}

Bandit problems are an online-learning framework in which a learner
sequentially selects from alternatives with uncertain rewards so as to
maximize cumulative reward
\cite{Lattimore2020Bandit}. %Representative methods include
%Upper Confidence Bound (UCB) algorithms %, which explore through optimistic confidence bounds 
%\cite{Auer2002UCB} %, 
%and Thompson sampling %, which explores by drawing from posterior distributions 
%\cite{Thompson1933TS}. 
In ranking bandits \cite{Kveton2015Cascading,Lagree2016PBM}, a learner sequentially selects a ranked item subset.
%extends these ideas to ranked presentation, in which multiple items are displayed simultaneously.

%\subsubsection{Click Models and Ranking Bandits:}

Click logs in search and recommendation are affected not only by the true
attractiveness of an item but also by its displayed position and the user
interface. Click models therefore describe click probabilities through a
generative model. Position bias and its modeling have long been studied
\cite{Craswell2008PositionBias}; a systematic account of click models is
given by Chuklin et al.~\cite{Chuklin2015ClickModels}. Ranking bandits based
on these models have subsequently been developed for online ranking.

%\subsubsection{Ranking Bandits Based on the Cascade Model:}
%\label{sec:cascade}

The cascade model assumes that a user examines items sequentially from the
top and stops browsing after clicking the first attractive item. Kveton
et al.\ formulated a ranking bandit under the cascade model and gave
learning algorithms and regret analyses based on partial feedback observed
up to the first click \cite{Kveton2015Cascading}. %Our multi-click setting
%includes a cascade-style heuristic that instead uses the last click as an
%inferred exposure boundary. 
Our model deals with carousel logs, which can reveal the actual browsing
depth, making it possible to exclude unviewed positions and estimate
attractiveness only from confirmed exposures.

%\subsubsection{Multiple-Play Bandits Based on the Position-Based Model:}

The position-based model (PBM) represents the click probability at position
\(j\) as the product of a position-dependent examination probability and
item attractiveness. Lagr\'ee et al.\ studied multiple-slot ranking bandits
under the PBM \cite{Lagree2016PBM}, and Komiyama et al.\ provided algorithms
and a theoretical analysis for the setting in which position bias is also
unknown \cite{Komiyama2017PBMUnknownBias}. In these studies, examination is
generally unobserved, thus position bias and attractiveness must be separated
using click feedback alone. %The additional parameters can make learning more difficult. 
Our work instead focuses on situations in which browsing
depth is observable: exposure is determined directly from depth rather than
being inferred as a latent variable.

\begin{table}[t]
\caption{Comparison of ranking-feedback models.}
\label{tab:model-comparison}
\centering
\scriptsize
\setlength{\tabcolsep}{3pt}
\begin{tabular}{p{0.20\textwidth}p{0.22\textwidth}p{0.22\textwidth}p{0.25\textwidth}}
\hline
Property & Cascade & PBM & Observable depth\\
\hline
Depth & Latent & Latent & Observed as \(V(t)\)\\
Click feedback & Up to first click & Per position & At displayed positions\\
Unseen positions & Partially observed & Censored & Explicitly excluded\\
Estimated quantities & Attractiveness & Attractiveness and bias & Attractiveness\\
Update rule & Partial feedback & Bias correction & Exposure based\\
\hline
\end{tabular}
\end{table}

%\subsubsection{Carousel Interfaces and Multi-List Presentation:}
%\label{sec:carousel}

%Modern recommender systems commonly display multiple recommendation lists
%as carousels rather than a single ranking. User behavior then spans two
%dimensions and can exhibit effects that are difficult to capture with
%single-list click models. 
Rahdari et al.\ proposed a click model that extends
ranked lists to carousel interfaces \cite{Rahdari2022CarouselClickModel} and
also investigated simulation-based evaluation for such interfaces
\cite{Rahdari2024SimCarousel}. Bendada et al.\ treated music-streaming
carousel optimization as a contextual-bandit problem and demonstrated its
effectiveness on real data \cite{Bendada2020CarouselPersonalization}. These
studies mainly emphasize behavior modeling and evaluation. We instead
formulate online ranking under information that is often available in real
systems---the observed browsing depth---and provide algorithms and regret
guarantees.

%\subsubsection{Position of This Work:}

Building on the established cascade and PBM ranking-bandit frameworks
\cite{Kveton2015Cascading,Lagree2016PBM,Komiyama2017PBMUnknownBias}, we focus
on the observability of browsing depth in carousel interfaces. Our objective
is to clarify what becomes more efficient than in conventional settings
that observe clicks alone by designing update rules based on confirmed
exposure and by incorporating observed depth into the theoretical analysis.
Table~\ref{tab:model-comparison} summarizes the distinction.

\section{Problem Setting}
\label{sec:problem}

Let \([K]:=\{1,\dots,K\}\) denote the set of items. Consider repeated
recommendations of \(L<K\) items. At each round \(t=1,2,\dots,T\), the learner
presents a ranking of \(L\) distinct items,
\[
\bm{a}(t)=(a_1(t),\dots,a_L(t))\in \Pi([K],L),
\]
where $\Pi([K],L)$ is the set of length-$L$ permutations of $[K]$.
For any \(L\)-item ranking \(\bm a=(a_1,\dots,a_L)\),
\(\langle\bm a\rangle\) denotes its item set:
\(\langle\bm a\rangle=\{a_1,\dots,a_L\}\). %The permutation $\bm{a}$ is also represented by the permutation function $\pi:[K]\rightarrow [K]$ that satisfies $\pi(a_i)=i$ for $i\in [L]$. Using such a permutation function $\pi$, $\bm{a}$ is represented by $(\pi^{-1}(1),\dots,\pi^{-1}(L))$.
The user certainly observes the first-ranked item and may then observe the
remaining items sequentially. The user stops browsing after observing the
\(V(t)\)-th ranked item, where \(V(t)\) is independent of the displayed
ranking. We call \(V(t)\in\{1,\dots,L\}\) the user's \emph{observation depth}
at round \(t\).
Let $P_j$ denote the probability with which the $j$th-ranked item is observed, that is,
\[
P_j:=\mathbbm{P}\{V(t)\ge j\}.
\]
We assume \(P_L>0\); thus \(1=P_1\ge P_2\ge\cdots\ge P_L>0\).
We also assume that \(V(1),V(2),\dots\) is an i.i.d.\ sequence.
The user is assumed to click each observed item
\(i\in\{a_1(t),a_2(t),\dots,a_{V(t)}(t)\}\) with unknown probability
\(\theta_i\).
Let $C_i(t)$ be the random variable that takes $1$ if item $i$ is clicked and $0$ otherwise at round $t$, that is 
\[
C_i(t)\sim\mathrm{Bernoulli}(\theta_i).
\]
For the theoretical analysis, we assume a unique preference order and,
without loss of generality, ids are assigned to the items so that
\[
\theta_1>\theta_2>\cdots>\theta_K,
\]
and the learner does not know this fact.
Let $\bm{a}^*=(1,2,\dots,L)$ be the optimal ranking and $\langle\bm{a}^*\rangle=[L]$.
The learner's reward for ranking \(\bm a(t)\) is the number of user clicks,
\(\sum_{j=1}^{V(t)}C_{a_j(t)}(t)\), whose expectation
\(\mu(\bm a(t))\) is
\begin{align*}
  \mu(\bm{a}(t))=&\mathbbm{E}\left[\sum_{j=1}^{V(t)}C_{a_j(t)}(t)\right]
  =\mathbbm{E}\left[\sum_{j=1}^{L}\mathbbm{1}\{V(t)\ge j\}C_{a_j(t)}(t)\right]\\
  =&\sum_{j=1}^{L}\mathbbm{E}\left[\mathbbm{1}\{V(t)\ge j\}\right]\mathbbm{E}\left[C_{a_j(t)}(t)\right]
  =\sum_{j=1}^LP_j\theta_{a_j(t)},
\end{align*}
where $\mathbbm{1}\{\mathcal{E}\}$ is $1$ if the event $\mathcal{E}$ occurs and $0$ otherwise for any event $\mathcal{E}$.
For $\bm{a}\in \Pi([K],L)$, define $\Delta_{\bm{a}}$ as $\Delta_{\bm{a}}:=\mu(\bm{a}^*)-\mu(\bm{a})$.
The $T$-round regret $\mathrm{Reg}_A(T)$ of the learner $A$ who makes ranking $\bm{a}(t)$ at round $t=1,2,\dots,T$,
is 
\[
\mathrm{Reg}_A(T)=T\mu(\bm{a}^*)-\sum_{t=1}^T \mu(\bm{a}(t))=\sum_{t=1}^T\Delta_{\bm{a}(t)}.
\]
The learner is evaluated by its expected regret $\mathbb{E}\left[\mathrm{Reg}_A(T)\right]$.
The Kullback--Leibler divergence between Bernoulli distributions with parameters $p$ and $q$ is denoted by
\[
\KL(p,q):=
 p\log\frac{p}{q}+(1-p)\log\frac{1-p}{1-q}.
 \]
For item $i>L$, $\Delta_i$ and $D_i$ are defined as
\[
\Delta_i:=\theta_{L}-\theta_i
\text{ and }
D_i:=\KL(\theta_i,\theta_{L}).
\]

\section{Proposed Methods}
\label{sec:method}

Define the item $i$'s exposure and click counts up to round \(t\) by
\begin{align*}
n_i(t)=\sum_{s=1}^t\sum_{j=1}^{V(s)}
\mathbbm 1[a_j(s)=i], \ \ 
s_i(t)=\sum_{s=1}^t\sum_{j=1}^{V(s)}
\mathbbm 1[a_j(s)=i]C_i(s). 
\end{align*}
For \(n_i(t)\ge1\), let $\hat{\theta}_i(t)$ be the mean of $\{C_i(s)\}_{s\le t,i\le V(s)}$, that is, 
\(\hat\theta_i(t)=s_i(t)/n_i(t)\). Because updates occur only after confirmed
exposure, these observations are direct Bernoulli samples.

The following subprocedure is used in all the proposed algorithms.\\
\ \\
\noindent\textbf{GetResponse}($\bm{a},\{n_i\}_{i\in[K]},\{s_i\}_{i\in [K]}$)
\begin{algorithmic}[1]
  \STATE $V,C_{a_1},\dots,C_{a_V}\gets$
  user's response for \(L\)-item ranking \(\bm{a}\)
  \STATE $(n'_i,s'_i)\gets \begin{cases} (n_i+1,s_i+C_i) & (i=a_1,\dots,a_V)\\(n_i,s_i) & (\text{otherwise})\end{cases}$
  \RETURN $\{n'_i\}_{i\in [K]},\{s'_i\}_{i\in [K]}$
\end{algorithmic}

\iffalse
\subsection{Exposure and Click Statistics}

Define the exposure count and click count of item \(i\) through round \(t\)
by
\begin{align}
n_i(t)
&:=
\sum_{s=1}^t
\sum_{j=1}^L
\mathbbm1[j\le V(s)]
\mathbbm1[a_j(s)=i],
\label{eq:n_i}\\
s_i(t)
&:=
\sum_{s=1}^t
\sum_{j=1}^L
\mathbbm1[j\le V(s)]
\mathbbm1[a_j(s)=i]C_i(s).
\label{eq:s_i}
\end{align}
For \(n_i(t)\ge1\), define the empirical attractiveness by
\[
\hat\theta_i(t)
:=
\frac{s_i(t)}{n_i(t)}.
\]
Because the update is performed only after confirmed exposure,
\(\hat\theta_i(t)\) is the empirical mean of direct Bernoulli samples from
\(\operatorname{Bernoulli}(\theta_i)\).
\fi

\subsection{Observable-Depth UCB}

The Observable-Depth UCB (Algorithm~\ref{alg:od-ucb}), OD-UCB for short, 
is a UCB algorithm for our problem.

At round \(t\), OD-UCB assigns item \(i\) the score
\[
\operatorname{UCB}_i(t)=
\begin{cases}
\hat\theta_i(t-1)+\sqrt{\alpha\log t/n_i(t-1)} &(n_i(t-1)\ge1)\\
+\infty,&(n_i(t-1)=0),
\end{cases}
\]
where \(\alpha>0\) is a constant parameter called \emph{exploration scale}. It presents the \(L\) items with the largest scores in
descending score order. %and updates only positions \(j\le V(t)\).

\begin{algorithm}[t]
\caption{Observable-Depth UCB}
\label{alg:od-ucb}
\begin{algorithmic}[1]
\REQUIRE %Number of items \(K\), slots \(L\), parameter 
$\alpha>0$
\STATE \(n_i(0)\leftarrow0,\ s_i(0)\leftarrow0\) for \(i=1,\ldots,K\)
\FOR{\(t=1,\ldots,T\)}
  \FOR{\(i=1,\ldots,K\)}
    \STATE \(\operatorname{UCB}_i\leftarrow
    \hat{\theta}_i(t-1)+\sqrt{\alpha\log t/n_i(t-1)}\) if \(n_i(t-1)\ge1\), and \(+\infty\) otherwise
  \ENDFOR
  \STATE \(\bm a(t) \gets\)  the top \(L\) items in descending order of $\operatorname{UCB}_i$
\STATE $\{n_i(t)\},\{s_i(t)\}\gets$ \textbf{GetResponse}($\bm{a}(t),\{n_i(t-1)\},\{s_i(t-1)\}$)
\iffalse
  \STATE Present \(\bm a(t)\); observe \(V(t)\) and clicks through \(V(t)\)
  \FOR{\(j=1,\ldots,V(t)\)}
    \STATE \(i\leftarrow a_j(t)\), \(n_i\leftarrow n_i+1\),
    \(s_i\leftarrow s_i+C_i(t)\)
  \ENDFOR
\fi
\ENDFOR
\end{algorithmic}
\end{algorithm}

\subsection{Observable-Depth Thompson Sampling}

The Observable-Depth Thompson Sampling algorithm
(Algorithm~\ref{alg:od-ts}), OD-TS for short, applies Thompson sampling to our
problem. It uses independent uniform
\(\operatorname{Beta}(1,1)\) priors for \(\theta_i\) (\(i\in[K]\)), yielding the
posterior
\(\operatorname{Beta}(1+s_i(t),1+n_i(t)-s_i(t))\) after \(t\) rounds.
\iffalse
OD-TS starts with independent priors
\(\theta_i\sim\operatorname{Beta}(a_0,b_0)\). Its posterior after round \(t\)
is
\[
\theta_i\mid\mathcal H_t\sim
\operatorname{Beta}(a_0+s_i(t),b_0+n_i(t)-s_i(t)).
\]
\fi
At each round \(t\), it draws \(\tilde\theta_i\) from
\(\operatorname{Beta}(1+s_i(t-1),1+n_i(t-1)-s_i(t-1))\) and presents the items $i$ with the $L$ largest $\tilde\theta_i$.
%in descending order.
%draws in descending order, and updates only confirmed exposures.

\begin{algorithm}[t]
\caption{Observable-Depth Thompson Sampling}
\label{alg:od-ts}
\begin{algorithmic}[1]
%\REQUIRE \(K,L\), and prior parameters \(a_0,b_0>0\)
\STATE \(n_i(0)\leftarrow0,\ s_i(0)\leftarrow0\) for \(i=1,\ldots,K\)
\FOR{\(t=1,\ldots,T\)}
  \FOR{\(i=1,\ldots,K\)}
    \STATE Draw
    \(\tilde\theta_i\sim\operatorname{Beta}(1+s_i(t-1),1+n_i(t-1)-s_i(t-1))\)
  \ENDFOR
  \STATE \(\bm a(t)\gets \) the top \(L\) items in descending order of $\tilde{\theta}$
  \STATE $\{n_i(t)\},\{s_i(t)\}\gets$ \textbf{GetResponse}($\bm{a}(t),\{n_i(t-1)\},\{s_i(t-1)\}$)
  \iffalse
  \STATE Present \(\bm a(t)\); observe \(V(t)\) and clicks through \(V(t)\)
  \FOR{\(j=1,\ldots,V(t)\)}
    \STATE \(i\leftarrow a_j(t)\), \(n_i\leftarrow n_i+1\),
    \(s_i\leftarrow s_i+C_i(t)\)
  \ENDFOR
  \fi
\ENDFOR
\end{algorithmic}
\end{algorithm}

\subsection{Observable-Depth DMED-L}

The Observable-Depth DMED-L algorithm
(Algorithm~\ref{alg:od-dmed-l}), OD-DMED-L for short, adapts DMED
\cite{Honda+Takemura:2010} to our problem. It uses an initially empty queue
\(Q\) to schedule exploration. After every item has been exposed at least once
(Lines~\ref{line:expose-begin}--\ref{line:expose-end}), the algorithm selects
the \(L\) items with the largest sample means when \(Q=\emptyset\). When
\(Q\neq\emptyset\), it dequeues the next ranking and uses it as \(\bm a(t)\).

At round \(t\) with \(Q=\emptyset\), each empirical non-top-\(L\) item \(i\)
is tested against
\[
 n_i(t-1)\KL(\hat\theta_i,\hat\theta_{a_L})
    <(1+\alpha)\log(t-1).
\]
For each item $i$ satisfying this inequality, the ranking
\((a_1(t),\dots,a_{L-1}(t),i)\) is enqueued in \(Q\).
\iffalse
OD-DMED-L adapts KL-based multiple-play exploration
\cite{Komiyama2017PBMUnknownBias} to observable-depth feedback. It replaces
index-based exploration with a KL-divergence test. Let
\(\pi_t\) be a permutation satisfying
\(\hat\theta_{\pi_t(1)}(t)\ge\cdots\ge
\hat\theta_{\pi_t(K)}(t)\), with ties resolved deterministically. The
empirical top-
\(L\) list is \(\hat{\bm a}^{\star}(t)=(\pi_t(1),\ldots,\pi_t(L))\).
For an empirical non-top item \(i\), the quantity
\[
n_i(t)\KL\!\left(\hat\theta_i(t),
\hat\theta_{\pi_t(L)}(t)\right)
\]
measures the evidence that \(i\) is worse than the current boundary item.
If this evidence is below \((1+\alpha)\log t\), the algorithm schedules an
exploratory list containing \(i\) in its final slot. A queue ensures that all
such candidates are tested once in the current batch.

During initialization, a cursor \(m\) records the number of exposures already
collected along a cyclic ordering. The displayed list is
\[
\operatorname{Cyclic}(m,L)
=\bigl(1+((m+r)\bmod K)\bigr)_{r=0}^{L-1}.
\]
After observing depth \(V(t)\), the cursor advances by \(V(t)\), so the next
list starts at the first item not exposed in the previous list. Because
\(V(t)\ge1\), every item is observed after at most \(K\) initialization
rounds. Algorithm~\ref{alg:od-dmed-l} gives the complete procedure.
\fi

\begin{algorithm}[t]
\caption{OD-DMED-L (Observable-Depth DMED-L)}
\label{alg:od-dmed-l}
\begin{algorithmic}[1]
\REQUIRE %Items \(K\), slots \(L\), horizon \(T\), parameter 
\(\alpha>0\)
\STATE \(n_i(0)\leftarrow0,\ s_i(0)\leftarrow0\) for \(i=1,\ldots,K\);
\(t\leftarrow1,\ m\leftarrow0\)
\WHILE{\(m<K\)}\label{line:expose-begin}
  \STATE \(\bm a(t)\leftarrow \bigl(1+((m+r)\bmod K)\bigr)_{r=0}^{L-1}\)
%  \operatorname{Cyclic}(m,L)\)
  \STATE $\{n_i(t)\},\{s_i(t)\}\gets$ \textbf{GetResponse}($\bm{a}(t),\{n_i(t-1)\},\{s_i(t-1)\}$)
  \iffalse
  \STATE Present \(\bm a(t)\); observe \(V(t)\) and clicks through \(V(t)\)
  \STATE Update \((n_{a_j(t)},s_{a_j(t)})\) for \(j\le V(t)\);
  \fi
 \STATE \(m\leftarrow m+V(t)\); \(t\leftarrow t+1\)
\ENDWHILE\label{line:expose-end}
\STATE \(Q\leftarrow\emptyset\)\label{algline:InitEnd}
\WHILE{\(t\le T\)}
  \IF{\(Q=\emptyset\)}
    %\STATE \(\hat\theta_i(t)\leftarrow s_i(t-1)/n_i(t-1)\) for all \(i\); sort items as \(\pi_t\)
    \STATE \(\bm a(t)\leftarrow\)
    the top $L$ items in descending order of $\hat{\theta}(t)$ \label{algline:a(t)}
    %(\pi_t(1),\ldots,\pi_t(L))\)
    \STATE \(S_t\leftarrow\{i: i\not\in \langle \bm{a}(t)\rangle,\
    n_i(t-1)\KL(\hat\theta_i,\hat\theta_{a_L})
    <(1+\alpha)\log(t-1)\}\)
    \FOR{\(i\in S_t\)}
      \STATE  \(Q.\operatorname{enqueue}((a_1(t),\ldots,a_{L-1}(t),i))\)
    \ENDFOR
  \ELSE
    \STATE \(\bm a(t)\leftarrow Q.\operatorname{dequeue}()\)
  \ENDIF
  \STATE $\{n_i(t)\},\{s_i(t)\}\gets$ \textbf{GetResponse}($\bm{a}(t),\{n_i(t-1)\},\{s_i(t-1)\}$)
  \iffalse
  \STATE Present \(\bm a(t)\); observe \(V(t)\) and clicks through \(V(t)\)
  \STATE Update \((n_{a_j(t)},s_{a_j(t)})\) for \(j\le V(t)\);
  \fi
 \STATE \(t\leftarrow t+1\)
\ENDWHILE
\end{algorithmic}
\end{algorithm}

\iffalse
OD-UCB and OD-TS require \(O(K\log L)\) work per round when top-
\(L\) selection uses a heap. OD-DMED-L sorts all items when starting a new
batch, costing \(O(K\log K)\), and uses \(O(L)\) work for a queued list.
All three methods require an additional \(O(V(t))\) update.
\fi

\section{Regret Analysis}
\label{sec:theory}

\subsection{Asymptotic Regret Lower Bound}
\label{sec:lower-bound}

\begin{definition}
A learner \(A\) is strongly consistent if
\[
\E[\Reg_A(T)]=o(T^a)
\quad\text{for any }a>0
\]
for every instance of the \(L\)-slot, \(K\)-armed ranking-bandit problem
with
\[
1\ge\theta_1>\theta_2>\cdots>\theta_K\ge0
\quad\text{and}\quad
1=P_1\ge P_2\ge\cdots\ge P_L>0.
\]
\end{definition}

\begin{theorem}
\label{thm:lower-bound}
For any strongly consistent learner \(A\),
\[
\liminf_{T\to\infty}
\frac{\E[\Reg_A(T)]}{\log T}
\ge
\sum_{i=L+1}^K
\frac{\Delta_i}
{\KL(\theta_i,\theta_L)}
\]
holds.
\end{theorem}
\begin{proof}
    See Appendix~\ref{app:sec:LB}.%\qed
\end{proof}

\iffalse
\begin{proof}
Since
\begin{align*}
\Delta_{\bm a(t)}
=
\sum_{j=1}^L
P_j(\theta_j-\theta_{i_j(t)})
\ge
P_L
\sum_{i_j(t)\in[K]\setminus[L]}
\Delta_{i_j(t)}
\end{align*}
holds,
\begin{align*}
\Reg_A(T)
=&
\sum_{t=1}^T\Delta_{\bm a(t)}\\
\ge&
P_L
\sum_{t=1}^T
\sum_{i_j(t)\in[K]\setminus[L]}
\Delta_{i_j(t)}\\
=&
P_L
\sum_{i=L+1}^K
\Delta_i
\sum_{t=1}^T
\mathbbm1\{i\in\bm a(t)\}.
\end{align*}
Thus,
\begin{align*}
\mathbb E[\Reg_A(T)]
\ge&
\sum_{i=L+1}^K
\Delta_i
\mathbb E\left[
P_L
\sum_{t=1}^T
\mathbbm1\{i\in\bm a(t)\}
\right]\\
=&
\sum_{i=L+1}^K
\mathbb E[n_i(T)]
\Delta_i.
\end{align*}
Hence, by Lemma~\ref{lem:lbn_i(T)},
\begin{align*}
\liminf_{T\to\infty}
\frac{\E[\Reg_A(T)]}{\log T}
\ge&
\sum_{i=L+1}^K
\Delta_i
\liminf_{T\to\infty}
\frac{\mathbb E[n_i(T)]}{\log T}\\
\ge&
\sum_{i=L+1}^K
\frac{\Delta_i}
{\KL(\theta_i,\theta_L)}.
\end{align*}
\end{proof}
\fi

\subsection{Finite-Time Upper Bound for OD-UCB}

\iffalse
This subsection gives an expected-regret upper bound for
\textbf{Observable-Depth UCB} (Algorithm~\ref{alg:od-ucb}). The proof follows
the classical UCB argument based on concentration inequalities and an upper
bound on the number of selections of suboptimal arms
\cite{Auer2002UCB,Lattimore2020Bandit}. In the present setting, however,
only items placed at positions \(j\le V(t)\) are exposed, and hence the
number of list inclusions and the number of exposures do not coincide.

Define the optimal set by
\[
\TopL:=\{(i)\in[K]\mid i=1,2,\dots,L\}.
\]
For the estimates \(\hat\theta_i(t):=s_i(t)/n_i(t)\), introduce the event
\[
\mathcal E_t
:=
\bigcap_{i=1}^K\left\{
\left|\hat\theta_i(t)-\theta_i\right|
\le
\sqrt{\frac{\alpha\log t}{n_i(t)}}\right\}.
\]
Its complement is denoted by \(\mathcal E_t^c\).
\fi

We have the following regret upper bound for OD-UCB.

\begin{theorem}
\label{thm:main}
Suppose \(\alpha>1/2\). For any integer \(1\le d\le L\),
\begin{align*}
\operatorname{Reg}_{\text{OD-UCB}}(T)\le{}&
32\alpha\log T
\left(
\frac{\left(\sum_{k=1}^L P_k\right)^2}{d}
+\left(\sum_{k=1}^d P_k\right)^2
\right)\\
&\quad\times
\left(
\sum_{i=L+1}^K
\frac{1}{P_L\Delta_i}
+\frac{L}{\Delta_{\bm a^{**}}}
\right)
+2LK\zeta(2\alpha),
\end{align*}
where \(\zeta(s):=\sum_{t=1}^{\infty}t^{-s}\) and
\(\bm{a}^{**}\in
\mathop{\arg\min}_{\bm{a}\in\Pi([K],L):\,\Delta_{\bm a}>0}
\Delta_{\bm{a}}\).
\end{theorem}
\begin{proof}
 See Appendix~\ref{app:sec:UCB_UB}. %\qed   
\end{proof}

\begin{remark}
\label{subsec:compare-pbmucb}
Lagr\'ee et al.~\cite{Lagree2016PBM} studied a multiple-play ranking
bandit under the position-based model and established a logarithmic regret
upper bound for PBM-UCB when the examination probabilities \(P_j\) are
known. According to their Theorem~9, for every \(\varepsilon>0\), there is
a constant \(C_0(\varepsilon)\) such that, for every integer
\(1\le d\le L\),
\begin{align}
\operatorname{Reg}(T)
\le{}&
16(1+\varepsilon)\log T
\left(
\frac{\left(\sum_{k=1}^LP_k\right)^2}{d}
+\left(\sum_{k=1}^dP_k\right)^2
\right)/P_L^2
\nonumber\\
&\times
\left(
\sum_{i=L+1}^K
\frac{1}{P_L\Delta_i}
+\frac{L}{\Delta_{\bm a^{**}}}
\right)+C_0(\varepsilon).
\label{eq:pbm-ucb-th9}
\end{align}
Comparing the limit \(\varepsilon\to0\) in this expression with the limit
\(\alpha\to1/2\) in Theorem~\ref{thm:main}, the leading term of the PBM-UCB
bound is larger by the factor \(P_L^{-2}\). Thus, the bound can deteriorate
substantially when the examination probability \(P_L\) of the last displayed
position is small.
\begin{remark}
SInce $\sum_{i=L+1}^K
\frac{\Delta_i}
{\KL(\theta_i,\theta_L)}\le \sum_{i=L+1}^K\frac{1}{2\Delta_i}$ holds by Pinsker's Inequality, the coefficient of $\log T$ of this bound is not asymptotically optimal compared to the lower bound shown by Theorem~\ref{thm:lower-bound} even if $\alpha \downarrow 1/2$.   
\end{remark}

\iffalse
%言っていることがよくわからないので一旦コメントアウトします。
In the present model, the browsing depth \(V(t)\) is observed in every
round, so whether position \(j\) was exposed, namely whether \(V(t)\ge j\),
is known directly. This corresponds to an uncensored rather than censored
PBM, for which Lagr\'ee et al. also observed a simplification of the
information structure.
\fi
\end{remark}

\subsection{Asymptotic Upper Bound for OD-DMED-L}

\iffalse
For the following proof, we retain the notation of the original
OD-DMED-L analysis: write \(\bm a(t)=\bm a(t)\),
\(\bm a^*=\bm a^*\), and let \(i_t^{-1}(r)=\pi_t(r)\).
Let \(T_0\) be the time \(t\) at which
Line~\ref{algline:InitEnd} is executed.
Let
\(\hat{\bm a}^*(t)=(i_t^{-1}(1),\dots,i_t^{-1}(L))\).
\fi

The regret for OD-DMED-L is upper bounded asymptotically as follows.

\begin{theorem}[Asymptotic regret upper bound of OD-DMED-L]
\label{thm:upper-bound}
For any parameter \(\alpha>0\), the following inequality holds:
\[
\limsup_{T\to\infty}
\frac{\E[\Reg_{\text{OD-DMED-L}}(T)]}{\log T}
\le
(1+\alpha)
\sum_{j=L+1}^K
\frac{\Delta_j}
{\KL(\theta_j,\theta_L)}.
\]
\end{theorem}
\begin{proof}
    See Appendix~\ref{app:sec:DMED_UB}.%\qed
\end{proof}

This upper bound coincides with the lower
bound shown by Theorem~\ref{thm:lower-bound} as \(\alpha\downarrow0\). Thus, OD-DMED-L is asymptotically optimal in
this limit.

\section{Experiments}
\label{sec:experiments}

We evaluate OD-UCB, OD-TS, and OD-DMED-L against baselines in synthetic simulations and in a
semi-synthetic study parameterized from RecGaze logs.
%In every experiment the environment follows the cutoff model in Sect.~\ref{sec:problem}. 
In our experimental setting, proposed methods observe \(V(t)\), whereas baselines
receive the same clicks but not the realized depth. Within each seed, all
policies share the depth and click random variates. Performance of each algorithm A is evaluated by \(Reg_A(T)=\sum_{t=1}^T\Delta_{\bm a(t)}\).
Thus, the plotted quantity accumulates expected-reward gaps, and averaging it
over seeds estimates \(\E[\Reg_A(T)]\).

\subsection{Settings}

For the finite-horizon comparisons, we use \(K=50\), \(L=5\), and
\(T=500{,}000\). For exposure probabilities \(P_1,\dots,P_L\) and click
probabilities \(\theta_1,\dots,\theta_K\), we consider the following three
cases:\\

{\small
\noindent
\begin{tabular}{l|l|p{4.3cm}}
environment & $P_1,\dots,P_5$ & $\theta_1,\dots,\theta_{50}$\\
\hline
synthetic shallow & $1.00,0.55,0.30,0.15,0.08$ &\multirow{2}{4.3cm}{$0.18,0.16,0.14,0.12,0.10,$ $0.09-0.07i/44$ ($i=0,\dots,44$)}\\
\cline{1-2}
synthetic deep & $1.00,0.80,0.70,0.60,0.50$ &\\
\hline
RecGaze-driven & $1.0000,0.9031,0.8529,0.7720,0.6879$ & click-through rates estimated from RecGaze dataset
\end{tabular}\\
}

The RecGaze-driven environment is based on the RecGaze dataset,
which records eye tracking, clicks, cursor movement, and explanations for
carousel interactions. It contains 3,477 interactions from 87 participants
across three movie-selection tasks \cite{de-leon-martinez-2025-recgaze}.
Processing its gaze, cursor, and click records produces 23,075 usable
user--task--carousel observations. We retain the 50 items with the largest
exposure counts, each of which has 85--87 confirmed exposures, and use the
first five carousel positions. 
The exposure probabilities and click-through rates are estimated from these
records. The click-through rate of item \(i\) is Laplace smoothed as
\((c_i+1)/(n_i+2)\), where \(n_i\) and \(c_i\) are its numbers of exposures
and clicks.
Although some suboptimal estimates are tied, the boundary values
\(\theta_{5}=0.0909>\theta_{6}=0.0805\) give a unique optimal Top-\(L\)
set\footnote{The orders of item ids in the experiments are not descending in their click rates. Thus, $\theta_i$ here means not the click rate of item $i$ but the $i$th largest click rate among those of items.}.%; the strict preference-order assumption is used only in the theory.
\addtocounter{footnote}{-1}

%\subsection{Synthetic Environments}

%For \(L=5\), we use two exposure profiles:
%\[
%\text{shallow: }P=(1.00,0.55,0.30,0.15,0.08),\ \ 
%\text{deep: }P=(1.00,0.80,0.70,0.60,0.50).
%\]
Observation depth is sampled according to
\(\Pr(V(t)=j)=P_j-P_{j+1}\) (\(P_{L+1}=0\)). %We set \(K=50\), \(L=5\), and \(T=500{,}000\). The top five attractiveness values are \(0.18,0.16,0.14,0.12,0.10\); the remaining 45 values are evenly spaced from 0.09 to 0.02 and randomly assigned to items with a fixed seed. 
The synthetic attractiveness values are randomly assigned to
implementation-level item IDs using seed 123, shared by all algorithms\footnotemark. %; the analytical ordering above is only a without-loss-of-generality labeling.

As baseline algorithms, we use PBM-UCB and PBM-TS~\cite{Lagree2016PBM},
which are supplied with the position biases \(P_j\) but not the realized
depth, and LastClick-UCB, a heuristic that extends the cascade-model feedback
rule~\cite{Kveton2015Cascading} to multiple-click observations by treating
positions through the last click as exposed. It retains all clicks in this
prefix and treats the entire list as exposed when no click occurs.
Refer the detailed ranking policy of above three algorithms to Supplementary Material~\ref{app:seccomp}.

OD-UCB and LastClick-UCB use exploration scale \(\alpha=1.0\). PBM-UCB uses %deviation scale \(2.0\), offsetting the factor \(1/2\) in its confidence radius and matching the leading exploration coefficient. 
$\delta(t)=2\log t$ in the calculation of $\mathrm{UCB}_i^{\mathrm{PBM}(t)}$ (Eq.~(\ref{eq:PBMUCB}) in Appendix~\ref{app:seccomp}).
OD-DMED-L uses
\(\alpha=0.1\).
OD-TS draws directly from the posterior of the conjugate prior \(\operatorname{Beta}(1,1)\),
whereas PBM-TS uses rejection sampling from the exact non-conjugate PBM
posterior with known position biases.
Initialization regret is included for every method. The results are averaged
over seeds 0--9 and 0--29 in the synthetic and RecGaze-driven environments.
Average regret are calculated with pointwise normal-approximation 95\% confidence intervals, which are the sample mean
plus or minus \(1.96\) standard errors.

\subsection{Results and Discussion}

\begin{figure}[t]
\centering
\includegraphics[width=0.49\textwidth]{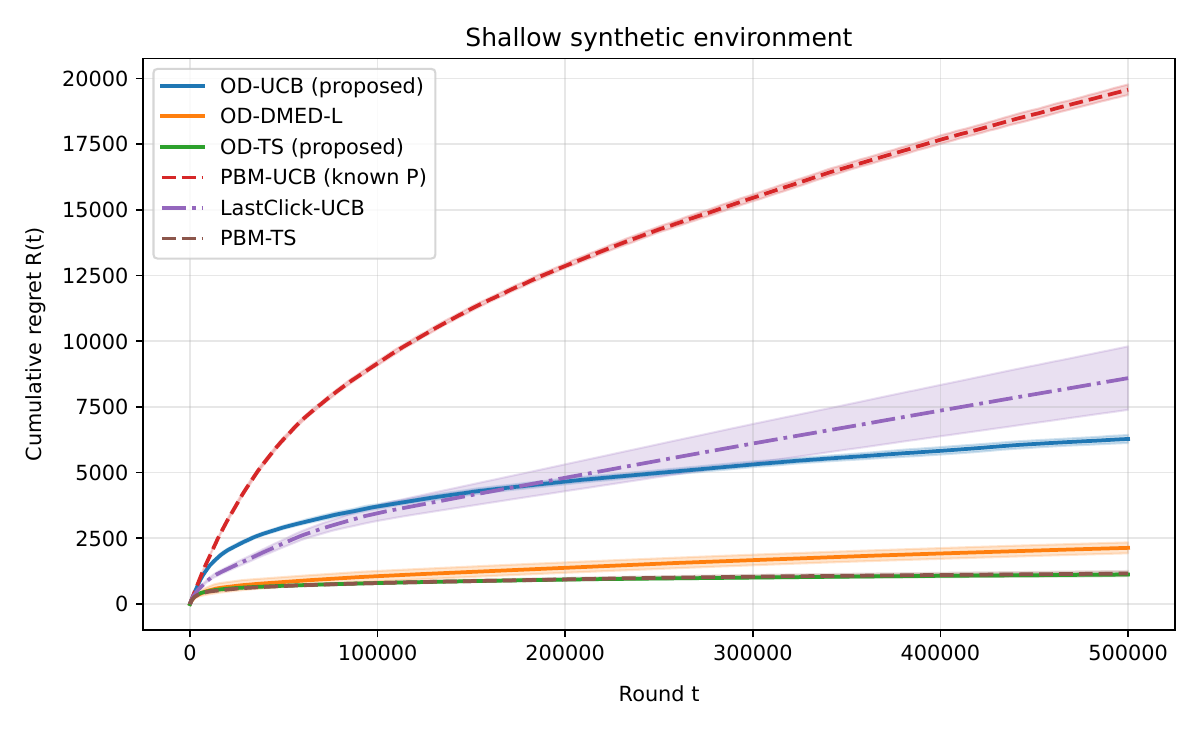}
\includegraphics[width=0.49\textwidth]{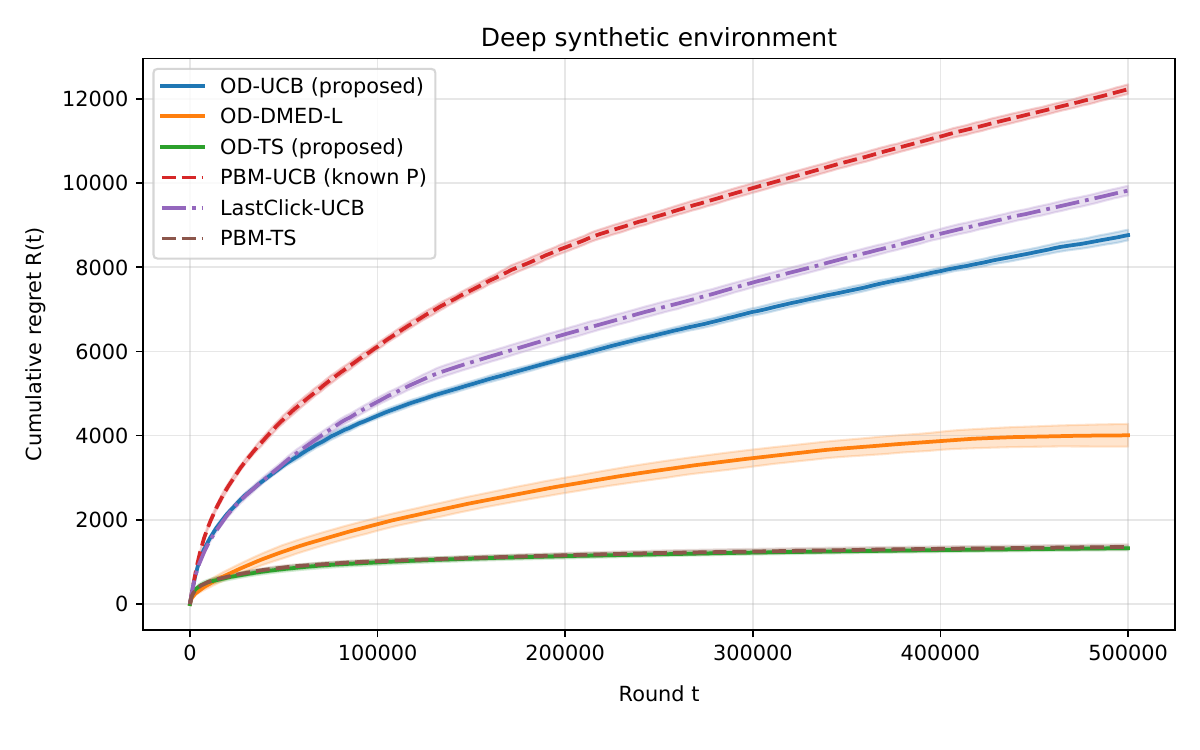}
\caption{Cumulative regrets in the shallow (left) and deep (right) synthetic environments.
Curves are means over ten seeds; pointwise 95\% confidence intervals are shaded.}
\label{fig:exp_shallow}
\end{figure}
\iffalse
\begin{figure}[t]
\centering
\includegraphics[width=\textwidth]{regret_deep.pdf}
\caption{Cumulative regret in the deep synthetic environment.
Curves are means over ten seeds; shaded regions are 95\% confidence
intervals.}
\label{fig:exp_deep}
\end{figure}
\fi

\begin{figure}[t]
\centering
\includegraphics[width=0.49\textwidth]{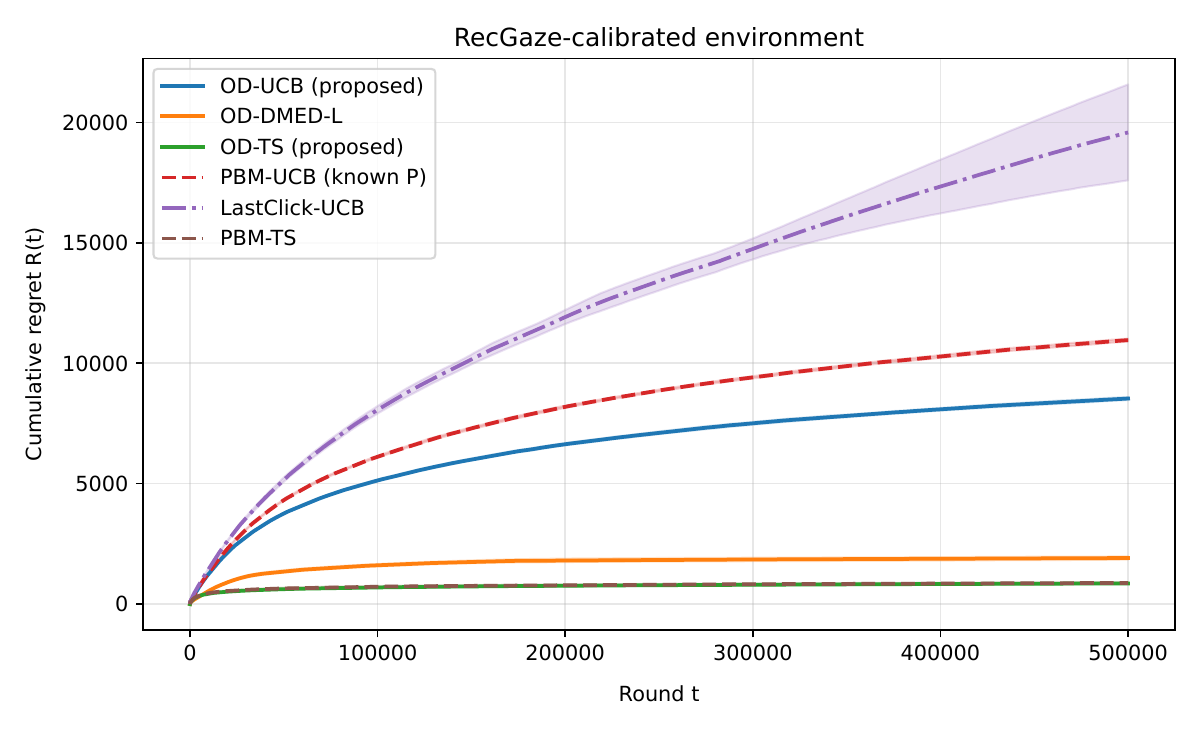}
\caption{Cumulative regret in the RecGaze-calibrated environment.
The curve is a mean over 30 seeds; pointwise 95\% confidence intervals are shaded.}
\label{fig:exp_recgaze}
\end{figure}

\begin{figure}[t]
\centering
\includegraphics[width=\textwidth]{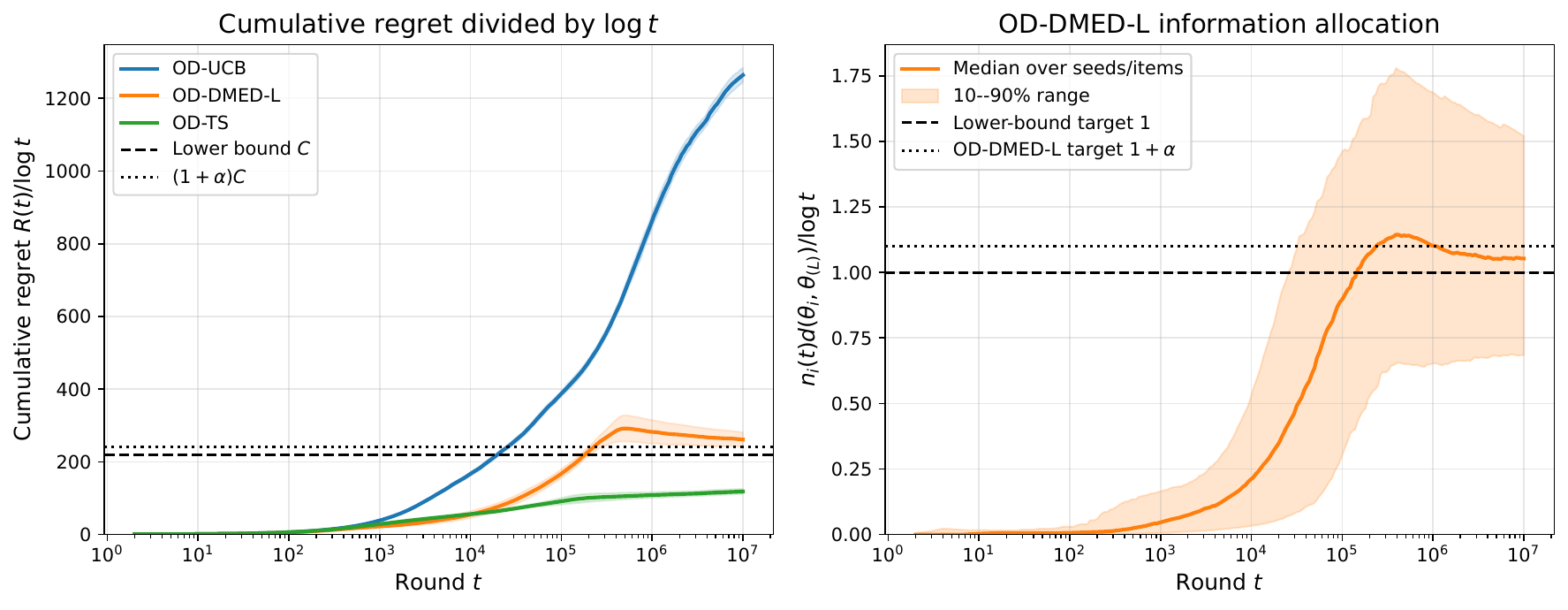}
\caption{Asymptotic behavior in the deep synthetic environment. Left:
mean $\Reg_{A}(t)/\log t$ with pointwise 95\% confidence intervals. Right: median
 $n_i(t)\KL(\theta_i,\theta_{(L)})/\log t$ of OD-DMED-L over seeds and suboptimal items, with the
10--90\% range.}
\label{fig:exp_asymptotic}
\end{figure}

At \(T=500{,}000\), OD-TS and PBM-TS have mean regrets \(1{,}121.8\) and
\(1{,}159.9\) in the shallow environment, and \(1{,}326.9\) and \(1{,}356.8\)
in the deep environment; their confidence intervals overlap (see
Fig.~\ref{fig:exp_shallow}). OD-DMED-L is third at \(2{,}137.2\) and
\(4{,}006.6\). The OD-UCB, PBM-UCB, and LastClick-UCB values are
\(6{,}280.6\), \(19{,}568.2\), and \(8{,}594.8\) in the shallow environment,
and \(8{,}761.5\), \(12{,}226.3\), and \(9{,}819.1\) in the deep environment.

The RecGaze-calibrated experiment gives the same ordering at the final horizon:
OD-TS has final mean regret \(856.1\), followed by PBM-TS at \(876.9\) and
OD-DMED-L at \(1{,}910.5\) (see Fig.~\ref{fig:exp_recgaze}). OD-UCB,
PBM-UCB, and LastClick-UCB reach \(8{,}532.1\), \(10{,}958.5\), and
\(19{,}580.5\), respectively. OD-TS and PBM-TS have similar final mean regret
on this instance, and their confidence intervals overlap. Unlike PBM-TS,
OD-TS requires neither the true \(P_j\) nor rejection sampling from a
non-conjugate posterior, thereby avoiding its additional posterior-sampling
overhead.

All three observable-depth methods condition their updates on realized
exposure. At the final horizon in all three environments, they achieve lower
mean regret than PBM-UCB and LastClick-UCB. The shallow--deep synthetic
comparison is consistent with a larger benefit from depth observation when
exposure falls more rapidly with position. The RecGaze instance has an even
deeper exposure profile than the deep synthetic instance, so its performance
differences cannot be attributed to browsing depth alone. OD-TS also has the
strongest finite-horizon performance among the proposed methods on these
instances.

The RecGaze result is model-driven rather than an online deployment and does
not establish regret under the true user-behavior process. The unusually high
click rates of a few logged items may also reflect task-specific intent rather
than context-free attractiveness. Nevertheless, agreement with the synthetic
results suggests that the same qualitative ordering relative to PBM-UCB and
LastClick-UCB can occur on this single instance calibrated from real
interaction logs.

To examine the asymptotic behavior directly, we record the two ratios in
Fig.~\ref{fig:exp_asymptotic} at 222 logarithmically spaced checkpoints.
On the deep synthetic instance we run OD-UCB, OD-TS, and OD-DMED-L for
\(T=10^7\) over ten seeds. %For every suboptimal item, we monitor
%\[
%\frac{n_i(t)\KL(\theta_i,\theta_L)}{\log t},
%\]
%whose lower-bound target is one and whose OD-DMED-L threshold is \(1+\alpha=1.1\). 

The results for the deep instance are shown in Fig.~\ref{fig:exp_asymptotic}.
The reference lines are asymptotic coefficients rather than finite-time
bounds; in particular, a finite-horizon curve below \(C\) does not contradict
Theorem~\ref{thm:lower-bound}. For the deep synthetic instance, the coefficient in
Theorem~\ref{thm:lower-bound} is \(\sum_{i=L+1}^K
\frac{\Delta_i}
{\KL(\theta_i,\theta_{(L)})}=218.45\), and the OD-DMED-L reference is
\((1+\alpha)\sum_{i=L+1}^K
\frac{\Delta_i}
{\KL(\theta_i,\theta_{(L)})}=240.30\). The median information allocation is
1.102 near \(T=10^6\), close to the OD-DMED-L threshold
\(1+\alpha=1.1\), and 1.053 at \(T=10^7\), between the threshold and the
lower-bound target 1.
\(\Reg_{\text{OD-DMED-L}}(t)/\log t\) peaks at 291.78 near
\(t=5.09\times10^5\) and decreases to 261.30 at \(T=10^7\), moving toward
the reference value 240.30 over the measured horizon. This ratio can decrease
because its denominator grows with \(t\), although the cumulative regret
itself is nondecreasing.

\FloatBarrier
\section{Conclusion}
\label{sec:conclusion}

We formulated a ranking-bandit problem with observable browsing depth and
proposed three algorithms that update only from confirmed exposures. We
established a finite-time, gap-dependent logarithmic regret upper bound for
OD-UCB and an asymptotic upper bound for OD-DMED-L that coincides with the
lower bound as \(\alpha\downarrow0\), establishing asymptotic optimality in
this limit. Across the three environments,
OD-TS attained final mean regret similar to PBM-TS, with overlapping
confidence intervals, although PBM-TS was supplied with the true marginal
examination probabilities. All proposed methods had lower final mean regret
than PBM-UCB and LastClick-UCB, and the long-horizon OD-DMED-L diagnostics
moved toward the predicted range. The evaluations assume independent clicks
and ranking-independent depth; RecGaze is log-calibrated rather than an
online deployment. Future work includes relaxing these assumptions, online
validation, sharper OD-UCB analysis, and finite-time guarantees for
OD-DMED-L and OD-TS.

\iffalse
\begin{credits}
\subsubsection{\ackname}
The authors thank the members of the Algorithms Laboratory for valuable
discussions and comments.

\subsubsection{\discintname}
The authors have no competing interests to declare that are relevant to the
content of this article.
\end{credits}
\fi

\bibliographystyle{plainnat}
\bibliography{reference}

\clearpage
\appendix

\section{Proof of Theorem~\ref{thm:lower-bound}}\label{app:sec:LB}

\begin{theorem}[
Bretagnolle--Huber inequality
(Theorem~14.2 in \cite{Lattimore2020Bandit})]
\label{th:BH}
Let \(\mathbb P\) and \(\mathbb Q\) be probability measures on the same
measurable space \((\Omega,\mathcal F)\), and let \(E\in\mathcal F\) be an
arbitrary event. Then
\[
\mathbb P(E)+\mathbb Q(E^c)
\ge
\frac{1}{2}e^{-D(\mathbb P,\mathbb Q)},
\]
where \(E^c=\Omega\setminus E\) is the complement of \(E\), and
\(D(\mathbb P,\mathbb Q)\) is the KL divergence of \(\mathbb P\) from
\(\mathbb Q\).
\end{theorem}

\begin{lemma}
\label{lem:lbn_i(T)}
Any strongly consistent learner satisfies
\[
\liminf_{T\rightarrow\infty}
\frac{\mathbb E[n_i(T)]}{\log T}
\ge
\frac{1}{\KL(\theta_i,\theta_L)}
\]
for each \(i>L\).
\end{lemma}

\begin{proof}
Let
\(\bm\theta=(\theta_1,\dots,\theta_K)\) and
\(\bm P=(P_1,\dots,P_L)\).
Define
\[
\bm\theta'
=
(\theta_1,\dots,\theta_{i-1},
\theta_L+\epsilon,
\theta_{i+1},\dots,\theta_K)
\]
for sufficiently small \(\epsilon>0\).

Let
\[
O_t
=
(V(t),C_{i_1(t)}(t),\dots,C_{i_{V(t)}}(t))
\]
and \(O_{1:T}=(O_1,O_2,\dots,O_T)\). Define the event
\[
E=\{n_i(T)>T/2\}.
\]
Then, by Theorem~\ref{th:BH},
\[
\mathbb P_{\bm\theta,\bm P}(E)
+\mathbb P_{\bm\theta',\bm P}(E^c)
\ge
\frac{1}{2}
e^{-D(
\mathbb P_{\bm\theta,\bm P},
\mathbb P_{\bm\theta',\bm P}
)}.
\]
Since
\begin{align*}
&D(
\mathbb P_{\bm\theta,\bm P},
\mathbb P_{\bm\theta',\bm P}
)\\
=&
\mathbb E_{\bm\theta,\bm P}
\left[
\log
\frac{
\mathbb P_{\bm\theta,\bm P}(O_{1:T})
}{
\mathbb P_{\bm\theta',\bm P}(O_{1:T})
}
\right]\\
=&
\mathbb E_{\bm\theta,\bm P}
\left[
\log
\left(
\frac{\theta_i}{\theta_i'}
\right)^{s_i(T)}
\left(
\frac{1-\theta_i}{1-\theta_i'}
\right)^{n_i(T)-s_i(T)}
\right]\\
=&
\mathbb E_{\bm\theta,\bm P}
\left[
s_i(T)\log\frac{\theta_i}{\theta_i'}
+(n_i(T)-s_i(T))
\log\frac{1-\theta_i}{1-\theta_i'}
\right]\\
=&
\mathbb E_{\bm\theta,\bm P}[s_i(T)]
\log\frac{\theta_i}{\theta_i'}
+\mathbb E_{\bm\theta,\bm P}[n_i(T)-s_i(T)]
\log\frac{1-\theta_i}{1-\theta_i'}\\
=&
\mathbb E_{\bm\theta,\bm P}[n_i(T)]
\theta_i\log\frac{\theta_i}{\theta_i'}
+\mathbb E_{\bm\theta,\bm P}[n_i(T)]
(1-\theta_i)
\log\frac{1-\theta_i}{1-\theta_i'}\\
=&
\mathbb E_{\bm\theta,\bm P}[n_i(T)]
\KL(\theta_i,\theta_i'),
\end{align*}
we have
\[
\mathbb P_{\bm\theta,\bm P}(E)
+\mathbb P_{\bm\theta',\bm P}(E^c)
\ge
\frac{1}{2}
e^{
-\mathbb E_{\bm\theta,\bm P}[n_i(T)]
\KL(\theta_i,\theta_i')
}.
\]

Let \(\Reg'(T)\) be the regret \(\Reg(T)\) of a learner in the
\(\bm\theta'\) environment. Then,
\begin{align*}
\Reg(T)+\Reg'(T)
\ge&
\frac{T}{2}
\left(
\mathbb P_{\bm\theta,\bm P}(E)P_L\Delta_i
+\mathbb P_{\bm\theta',\bm P}(E^c)P_L\epsilon
\right)\\
\ge&
\frac{TP_L\min\{\Delta_i,\epsilon\}}{2}
\left(
\mathbb P_{\bm\theta,\bm P}(E)
+\mathbb P_{\bm\theta',\bm P}(E^c)
\right)\\
\ge&
\frac{TP_L\min\{\Delta_i,\epsilon\}}{4}
e^{
-\mathbb E_{\bm\theta,\bm P}[n_i(T)]
\KL(\theta_i,\theta_i')
}.
\end{align*}
Rearranging this inequality gives
\[
\mathbb E_{\bm\theta,\bm P}[n_i(T)]
\ge
\frac{1}{d(\theta_i,\theta_i')}
\log
\frac{
TP_L\min\{\Delta_i,\epsilon\}
}{
4(\Reg(T)+\Reg'(T))
}.
\]
Therefore,
\begin{align*}
&\liminf_{T\rightarrow\infty}
\frac{
\mathbb E_{\bm\theta,\bm P}[n_i(T)]
}{
\log T
}\\
\ge&
\frac{1}{d(\theta_i,\theta_i')}
\liminf_{T\rightarrow\infty}
\frac{
\log
\frac{
TP_L\min\{\Delta_i,\epsilon\}
}{
4(\Reg(T)+\Reg'(T))
}
}{
\log T
}\\
=&
\frac{1}{d(\theta_i,\theta_i')}
\left(
1-
\limsup_{T\rightarrow\infty}
\frac{
\log(\Reg(T)+\Reg'(T))
}{
\log T
}
\right)\\
=&
\frac{1}{d(\theta_i,\theta_i')},
\end{align*}
where
\[
\limsup_{T\rightarrow\infty}
\frac{\log(\Reg(T)+\Reg'(T))}{\log T}
=0
\]
holds by the strong-consistency assumption on the learner.
%\qed
\end{proof}

\begin{proof}[Proof of Theorem~\ref{thm:lower-bound}]
Since
\begin{align*}
\Delta_{\bm a(t)}
=
\sum_{j=1}^L
P_j(\theta_j-\theta_{a_j(t)})
\ge
P_L
\sum_{i_j(t)\in[K]\setminus[L]}
\Delta_{a_j(t)}
\end{align*}
holds,
\begin{align*}
\Reg_A(T)
=&
\sum_{t=1}^T\Delta_{\bm a(t)}\\
\ge&
P_L
\sum_{t=1}^T
\sum_{a_j(t)\in[K]\setminus[L]}
\Delta_{a_j(t)}\\
=&
P_L
\sum_{i=L+1}^K
\Delta_i
\sum_{t=1}^T
\mathbbm1\{i\in\bm a(t)\}.
\end{align*}
Thus,
\begin{align*}
\mathbb E[\Reg_A(T)]
\ge&
\sum_{i=L+1}^K
\Delta_i
\mathbb E\left[
P_L
\sum_{t=1}^T
\mathbbm1\{i\in\bm a(t)\}
\right]\\
=&
\sum_{i=L+1}^K
\mathbb E[n_i(T)]
\Delta_i.
\end{align*}
Hence, by Lemma~\ref{lem:lbn_i(T)},
\begin{align*}
\liminf_{T\to\infty}
\frac{\E[\Reg_A(T)]}{\log T}
\ge&
\sum_{i=L+1}^K
\Delta_i
\liminf_{T\to\infty}
\frac{\mathbb E[n_i(T)]}{\log T}\\
\ge&
\sum_{i=L+1}^K
\frac{\Delta_i}
{\KL(\theta_i,\theta_L)},
\end{align*}
which completes the proof of the theorem.
%\qed
\end{proof}

\section{Proof of Theorem~\ref{thm:main}}\label{app:sec:UCB_UB}

\iffalse
Define the optimal set by
\[
\TopL:=\{(i)\in[K]\mid i=1,2,\dots,L\}.
\]
\fi
For the estimates \(\hat\theta_i(t):=s_i(t)/n_i(t)\), introduce the event
\[
\mathcal E_t
:=
\bigcap_{i=1}^K\left\{
\left|\hat\theta_i(t-1)-\theta_i\right|
\le
\sqrt{\frac{\alpha\log t}{n_i(t-1)}}\right\}.
\]
Its complement is denoted by \(\mathcal E_t^c\).

\begin{lemma}
\label{lem:hoeffding-sum}
If \(\alpha>1/2\), then
\[
\mathbb E\left[\sum_{t=1}^T
\Delta_{\bm a(t)}\mathbbm1[\mathcal E_t^c]\right]
\le 2LK\sum_{t=1}^{\infty}t^{-2\alpha}
=2LK\zeta(2\alpha).
\]
\end{lemma}

\begin{proof}
Since \(\Delta_{\bm a}\le L\) for every ranking \(\bm a\),
\[
\mathbb E\left[\sum_{t=1}^T
\Delta_{\bm a(t)}\mathbbm1[\mathcal E_t^c]\right]
\le L\sum_{t=1}^T\Pr(\mathcal E_t^c).
\]
The probability \(\Pr(\mathcal E_t^c)\) is bounded by
\(2Kt^{-2\alpha}\) as follows:
\begin{align*}
\Pr(\mathcal E_t^c)
=&\Pr\left\{\bigcup_{i=1}^K
\left\{\left|\hat\theta_i(t-1)-\theta_i\right|
>
\sqrt{\frac{\alpha\log t}{n_i(t-1)}}\right\}\right\}\\
\le&\sum_{i=1}^K
\Pr\left\{\left|\hat\theta_i(t-1)-\theta_i\right|
>
\sqrt{\frac{\alpha\log t}{n_i(t-1)}}\right\}\\
=&\sum_{i=1}^K
\Pr\left\{
\left\{\hat\theta_i(t-1)>\theta_i+
\sqrt{\frac{\alpha\log t}{n_i(t-1)}}\right\}
\right.\\
&\hspace*{1.2cm}\left.
\cup
\left\{\hat\theta_i(t-1)<\theta_i-
\sqrt{\frac{\alpha\log t}{n_i(t-1)}}\right\}
\right\}\\
\le&\sum_{i=1}^K
\Pr\left\{\hat\theta_i(t-1)>\theta_i+
\sqrt{\frac{\alpha\log t}{n_i(t-1)}}\right\}\\
&+\sum_{i=1}^K
\Pr\left\{\hat\theta_i(t-1)<\theta_i-
\sqrt{\frac{\alpha\log t}{n_i(t-1)}}\right\}\\
\le&\sum_{i=1}^K
e^{-2n_i(t-1)\cdot\frac{\alpha\log t}{n_i(t-1)}}
+\sum_{i=1}^K
e^{-2n_i(t-1)\cdot\frac{\alpha\log t}{n_i(t-1)}}\\
=&2Ke^{-2\alpha\log t}
=2Kt^{-2\alpha},
\end{align*}
where the last inequality follows from Hoeffding's inequality. Hence
\[
\sum_{t=1}^T\Pr(\mathcal E_t^c)
\le2K\sum_{t=1}^{\infty}t^{-2\alpha}.
\]
The right-hand side converges for \(\alpha>1/2\), proving the lemma.
%\qed
\end{proof}

Define the event
\[
\mathcal F_t
=
\left\{
0<\Delta_{\bm a(t)}
\le2\sum_{j=1}^LP_j
\sqrt{\frac{\alpha\log t}{n_{a_j(t)}(t-1)}}
\right\}.
\]

\begin{lemma}
\label{lemma:F_t}
The following inclusion holds:
\[
\mathcal E_t\cap\{\Delta_{\bm a(t)}>0\}
\subseteq\mathcal F_t.
\]
\end{lemma}

\begin{proof}
Since ranking \(\bm a(t)\) is selected in round \(t\),
\[
\sum_{j=1}^LP_j\operatorname{UCB}_{a_j(t)}(t)
\ge
\sum_{j=1}^LP_j\operatorname{UCB}_j(t).
\]
Suppose that \(\mathcal E_t\) occurs. Then, for all \(j\in[K]\),
\[
\theta_j-\sqrt{\frac{\alpha\log t}{n_j(t-1)}}
\le\hat\theta_j(t-1)
\le\theta_j+\sqrt{\frac{\alpha\log t}{n_j(t-1)}}.
\]
Therefore,
\[
\sum_{j=1}^LP_j
\left(
\theta_{a_j(t)}
+2\sqrt{\frac{\alpha\log t}{n_{a_j(t)}(t-1)}}
\right)
\ge
\sum_{j=1}^LP_j\theta_j.
\]
It follows that
\[
\Delta_{\bm a(t)}
=
\sum_{j=1}^LP_j\theta_j
-\sum_{j=1}^LP_j\theta_{a_j(t)}
\le
2\sum_{j=1}^LP_j
\sqrt{\frac{\alpha\log t}{n_{a_j(t)}(t-1)}}.
\]
Thus, whenever \(\mathcal E_t\) and
\(\{\Delta_{\bm a(t)}>0\}\) both occur, \(\mathcal F_t\) occurs. %\qed
\end{proof}

Define the set of positions
\[
U_t
=
\left\{
j\in[L]\ \middle|\
n_{a_j(t)}(t-1)
\le
\frac{16\alpha(\sum_{k=1}^LP_k)^2\log t}
{\Delta_{\bm a(t)}^2}
\right\}.
\]
For an integer \(0<d<L\), consider the two events
\begin{align*}
\mathcal G_t(d)
=&\{|U_t|\ge d\},\\
\mathcal H_t(d)
=&
\left\{
|U_t|<d,\ %\exists j\in[L],
n_{a_j(t)}(t-1)
\le
\frac{16\alpha(\sum_{k=1}^dP_k)^2\log t}
{\Delta_{\bm a(t)}^2}
\text{ for some } j\in [L]
\right\}.
\end{align*}

\begin{lemma}
\label{lemma:GH}
For every integer \(0<d<L\),
\[
\mathcal F_t\subseteq
\mathcal G_t(d)\cup\mathcal H_t(d).
\]
\end{lemma}

\begin{proof}
We prove the statement by contradiction. Suppose that \(\mathcal F_t\)
holds but \(\mathcal G_t(d)\cup\mathcal H_t(d)\) does not. Then
\[
|U_t|<d,\qquad
n_{a_j(t)}(t-1)>
\frac{16\alpha(\sum_{k=1}^dP_k)^2\log t}
{\Delta_{\bm a(t)}^2} \text{ for all }j\in[L].
\]
Consequently,
\begin{align*}
\Delta_{\bm a(t)}
\le&
2\sum_{j=1}^LP_j
\sqrt{\frac{\alpha\log t}{n_{a_j(t)}(t-1)}}
\tag{by \(\mathcal F_t\)}\\
=&
2\sum_{j\in U_t}P_j
\sqrt{\frac{\alpha\log t}{n_{a_j(t)}(t-1)}}
+2\sum_{j\in[L]\setminus U_t}P_j
\sqrt{\frac{\alpha\log t}{n_{a_j(t)}(t-1)}}\\
<&
2\sum_{j\in U_t}P_j
\frac{\Delta_{\bm a(t)}}{4\sum_{k=1}^dP_k}
+2\sum_{j\in[L]\setminus U_t}P_j
\frac{\Delta_{\bm a(t)}}{4\sum_{k=1}^LP_k}\\
=&
\frac{\Delta_{\bm a(t)}}{2}
\left(
\frac{\sum_{j\in U_t}P_j}{\sum_{k=1}^dP_k}
+\frac{\sum_{j\in[L]\setminus U_t}P_j}
{\sum_{k=1}^LP_k}
\right)\\
\le&\Delta_{\bm a(t)},
\end{align*}
which is a contradiction. Hence
\(\mathcal F_t\subseteq\mathcal G_t(d)\cup\mathcal H_t(d)\).
%\qed
\end{proof}

Let \(\mathcal A_i\) denote the set of rankings containing item \(i\).
We use the following existing lemma to prove
Lemma~\ref{lemma:first_term}.

\begin{lemma}[Lemma~2 in Appendix B.4 of \cite{Combes2015}]
\label{lemma:Combes}
For every item \(i\in[K]\) and every \(C>0\),
\[
\sum_{t=1}^T
\mathbbm1\left[
i\in\bm a(t),\
n_i(t)\le\frac{C}{\Delta_{\bm a(t)}^2}
\right]\Delta_{\bm a(t)}
\le\frac{2C}{\Delta_{\bm{a}^{**}_i}},
\]
where
\(\bm{a}^{**}_i\in
\mathop{\arg\min}_{\bm{a}\in\Pi([K],L):\,\Delta_{\bm a}>0,i\in \bm{a}}
\Delta_{\bm{a}}\).
%where
%\[
%\Delta_{\min,i}
%=
%\min_{\bm a\in\mathcal A_i\setminus\{\bm a^*\}}
%\Delta_{\bm a}.
%\]
\end{lemma}

Further decompose the events \(\mathcal G_t(d)\) and
\(\mathcal H_t(d)\) as
\begin{align*}
\mathcal G_{i,t}(d)
=&\mathcal G_t(d)\cap\{i\in U_t\},\\
\mathcal H_{i,t}(d)
=&\mathcal H_t(d)\cap
\left\{
i\in\bm a(t),\
n_i(t)
\le
\frac{16\alpha(\sum_{k=1}^dP_k)^2\log t}
{\Delta_{\bm a(t)}^2}
\right\}.
\end{align*}
Then
\begin{align}
\sum_{i=1}^K\mathbbm1[\mathcal G_{i,t}(d)]
=\mathbbm1[\mathcal G_t(d)]
\sum_{i=1}^K\mathbbm1[i\in U_t]
\ge&d\mathbbm1[\mathcal G_t(d)],
\label{ineq:G_it}\\
\mathbbm1[\mathcal H_t(d)]
\le&\sum_{i=1}^K\mathbbm1[\mathcal H_{i,t}(d)].
\label{ineq:H_it}
\end{align}

\begin{lemma}
\label{lemma:first_term}
The following inequality holds:
\begin{align*}
\mathbb E\left[
\sum_{t=1}^T
\Delta_{\bm a(t)}\mathbbm1[\mathcal E_t]
\right]
\le&
32\alpha\log T
\left(
\frac{(\sum_{k=1}^LP_k)^2}{d}
+\left(\sum_{k=1}^dP_k\right)^2
\right)\\
&\times
\left(
\sum_{i=L+1}^K
\frac{1}{P_L\Delta_i}
+\frac{L}{\Delta_{\bm a^{**}}}
\right).
\end{align*}
\end{lemma}

\begin{proof}
The lemma follows from the following calculation:
\allowdisplaybreaks
\begin{align*}
&\mathbb E\left[
\sum_{t=1}^T
\Delta_{\bm a(t)}\mathbbm1[\mathcal E_t]
\right]\\
=&
\sum_{t=1}^T
\mathbb E\left[
\Delta_{\bm a(t)}
\mathbbm1[\mathcal E_t,\Delta_{\bm a(t)}>0]
\right]\\
\le&
\sum_{t=1}^T
\mathbb E\left[
\Delta_{\bm a(t)}\mathbbm1[\mathcal F_t]
\right]
\tag{by Lemma~\ref{lemma:F_t}}\\
\le&
\sum_{t=1}^T
\mathbb E\left[
\Delta_{\bm a(t)}
\left(
\mathbbm1[\mathcal G_t(d)]
+\mathbbm1[\mathcal H_t(d)]
\right)
\right]
\tag{by Lemma~\ref{lemma:GH}}\\
\le&
\sum_{t=1}^T
\mathbb E\left[
\Delta_{\bm a(t)}
\left(
\frac{1}{d}
\sum_{i=1}^K\mathbbm1[\mathcal G_{i,t}(d)]
+\sum_{i=1}^K\mathbbm1[\mathcal H_{i,t}(d)]
\right)
\right]
\tag{by Ineqs.~\eqref{ineq:G_it} and \eqref{ineq:H_it}}\\
=&
\mathbb E\left[
\frac{1}{d}
\sum_{i=1}^K\sum_{t=1}^T
\mathbbm1[\mathcal G_{i,t}(d)]
\Delta_{\bm a(t)}
+\sum_{i=1}^K\sum_{t=1}^T
\mathbbm1[\mathcal H_{i,t}(d)]
\Delta_{\bm a(t)}
\right]\\
\le&
\mathbb E\left[
\frac{1}{d}
\sum_{i=1}^K\sum_{t=1}^T
\mathbbm1\left[
i\in\bm a(t),\
n_i(t)
\le
\frac{16\alpha(\sum_{k=1}^LP_k)^2\log T}
{\Delta_{\bm a(t)}^2}
\right]\Delta_{\bm a(t)}
\right.\\
&\left.
\quad+
\sum_{i=1}^K\sum_{t=1}^T
\mathbbm1\left[
i\in\bm a(t),\
n_i(t)
\le
\frac{16\alpha(\sum_{k=1}^dP_k)^2\log T}
{\Delta_{\bm a(t)}^2}
\right]\Delta_{\bm a(t)}
\right]\\
\le&
32\alpha\log T
\left(
\frac{(\sum_{k=1}^LP_k)^2}{d}
+\left(\sum_{k=1}^dP_k\right)^2
\right)
\left(
\sum_{i=L+1}^K
\frac{1}{P_L\Delta_i}
+\frac{L}{\Delta_{\bm a^{**}}}
\right)
\tag{by Lemma~\ref{lemma:Combes} and $\bm{a}^{**}_i=\bm{a}^{**}$ if $i\in [L]$ and $\bm{a}^{**}_i=(1,\dots,L-1,i)$ otherwise}.
\end{align*}
%\qed
\end{proof}

\begin{proof}[Proof of Theorem~\ref{thm:main}]
Using the events \(\mathcal E_t\) and \(\mathcal E_t^c\), the regret is
decomposed as
\[
\operatorname{Reg}(T)
=
\mathbb E\left[\sum_{t=1}^T
\Delta_{\bm a(t)}\mathbbm1[\mathcal E_t]\right]
+\mathbb E\left[\sum_{t=1}^T
\Delta_{\bm a(t)}\mathbbm1[\mathcal E_t^c]\right].
\]
The first term is bounded in Lemma~\ref{lemma:first_term}, and the second
term is bounded in Lemma~\ref{lem:hoeffding-sum}. Combining those lemmas
gives the stated inequality.
%\qed
\end{proof}

\section{Proof of Theorem~\ref{thm:upper-bound}}\label{app:sec:DMED_UB}

We use the following notations:
\begin{align*}
\hat{\bm{a}}^*(t) :=& \bm{a}(t) \text{ calculated at Line~\ref{algline:a(t)} of Algorithm~\ref{alg:od-dmed-l}, that is, the top $L$ items}\\ 
&\text{in descending order of $\hat{\theta}(t)$ in the case with $Q=\emptyset$.}\\
T_b :=& \text{time $t$ of the $b$th execution of Line~\ref{algline:a(t)} of Algorithm~\ref{alg:od-dmed-l}.}\\
b_T :=& \max\{b\in[T]\mid T_b<T\}.\\
\mathcal{O}(t) :=& \left\{
j\in[K]\mid
j>L,\
\hat\theta_j(t)
\ge
\theta_L-\frac{\Delta_{L+1}}{2}
\right\}.\\
\mathcal{U}(t):= & \left\{
j\in[K]\mid
j\le L
\text{ with }
\hat\theta_j(t)
\le
\theta_L-\frac{\Delta_{L+1}}{2}
\right\}.\\
i_t(j) :=& \text{ descending-order rank of item $j$ in $\{\hat{\theta}_i(t)\}_{i\in [K]}$}. \\
i_t^{-1}(j):=& \text{item id of $j$th ranked item in $\{\hat{\theta}_i(t)\}_{i\in [K]}$}.
\end{align*}

\begin{lemma}
The following inequality holds:
\label{lem:topLdiff}
\begin{align*}
&\mathbb E\left[
\sum_{b=1}^{b_T}
\sum_{t=T_{b}}^{T_{b+1}-1}
\mathbbm1\left\{
\langle\hat{\bm a}^*(T_{b})\rangle\neq[L]
\right\}
\right]\le
\frac{K}{P_L}
\sum_{j=L+1}^K\frac{2}{\Delta_j^2}
+\frac{KL}{P_L}\frac{2}{\Delta_{L+1}^2}
+KL\zeta(1+\alpha),
\end{align*}
where $\zeta(s)=\sum_{t=1}^{\infty}t^{-s}$.
\end{lemma}

\begin{proof}
We have
\begin{align}
&\sum_{b=1}^{b_T}
\sum_{t=T_{b}}^{T_{b+1}-1}
\mathbbm1\left\{
\langle\hat{\bm a}^*(T_{b})\rangle\neq[L]
\right\}\nonumber\\
\le&
K\sum_{b=1}^{b_T}
\mathbbm1\left\{
\langle\hat{\bm a}^*(T_{b})\rangle\neq[L]
\right\}\nonumber\\
=&
K\sum_{b=1}^{b_T}
\mathbbm1\left\{
\langle\hat{\bm a}^*(T_{b})\rangle\neq[L],\
\mathcal O(T_{b})\neq\emptyset
\right\}\nonumber\\
&+
K\sum_{b=1}^{b_T}
\mathbbm1\left\{
\langle\hat{\bm a}^*(T_{b})\rangle\neq[L],\
\mathcal O(T_{b})=\emptyset,\
\mathcal U(T_{b})\neq\emptyset
\right\}. \label{eq:non-optimal set case}
\end{align}

The first term is bounded as
\begin{align*}
&\sum_{b=1}^{b_T}
\mathbbm1\left\{
\langle\hat{\bm a}^*(T_{b})\rangle\neq[L],\
\mathcal O(T_{b})\neq\emptyset
\right\}\\
\le&
\sum_{b=1}^{b_T}
\sum_{j=L+1}^K
\mathbbm1\left\{
\hat\theta_j(T_{b})
\ge\theta_L-\frac{\Delta_{L+1}}{2},\
j\in\langle\bm a(T_{b})\rangle
\right\}\\
\le&
\sum_{b=1}^{b_T}
\sum_{j=L+1}^K
\mathbbm1\left\{
\hat\theta_j(T_{b})
\ge\theta_j+\frac{\Delta_j}{2},\
j\in\langle\bm a(T_{b})\rangle
\right\}\\
=&
\sum_{b=1}^{b_T}
\sum_{j=L+1}^K
\sum_{s=1}^T
\mathbbm1\left\{
\hat\theta_j(T_{b})
\ge\theta_j+\frac{\Delta_j}{2},\
j\in\langle\bm a(T_{b})\rangle,\
n_j(T_{b}-1)=s
\right\}\\
=&
\sum_{j=L+1}^K
\sum_{s=1}^T
\sum_{b=1}^{b_T}
\mathbbm1\left\{
\hat\theta_j(T_{b})
\ge\theta_j+\frac{\Delta_j}{2},\
j\in\langle\bm a(T_{b})\rangle,\
n_j(T_{b}-1)=s
\right\}\\
\le&
\sum_{j=L+1}^K
\sum_{s=1}^T
Z_{j,s}
\mathbbm1\left\{
\hat\theta_{j,s}
\ge\theta_j+\frac{\Delta_j}{2}
\right\}.
\end{align*}
Here \(\hat\theta_{j,s}\) is a sample mean of size \(s\) drawn from
\(\mathrm{Bernoulli}(\theta_j)\), and \(Z_{j,s}\) is the number of
rounds \(T_{b}\) in which one of
\(\bm a(T_{b}),\dots,\bm a(T_{b+1}-1)\) includes item \(j\) during
\(n_j(T_{b}-1)=s\). Then,
\begin{align*}
&\mathbb E\left[
\sum_{b=1}^{b_T}
\mathbbm1\left\{
\langle\hat{\bm a}^*(T_{b})\rangle\neq[L],\
\mathcal O(T_{b})\neq\emptyset
\right\}
\right]\\
\le&
\sum_{j=L+1}^K
\sum_{s=1}^T
\mathbb E[Z_{j,s}]
\mathbb P\left\{
\hat\theta_{j,s}
\ge\theta_j+\frac{\Delta_j}{2}
\right\}\\
\le&
\sum_{j=L+1}^K
\sum_{s=1}^T
\frac{1}{P_L}e^{-s\frac{\Delta_j^2}{2}}
\tag{by Hoeffding's inequality}\\
\le&
\frac{1}{P_L}
\sum_{j=L+1}^K
\frac{1}{e^{\frac{\Delta_j^2}{2}}-1}
\le
\frac{1}{P_L}
\sum_{j=L+1}^K
\frac{2}{\Delta_j^2}.
\end{align*}

The second term of the right-hand side of Eq.~(\ref{eq:non-optimal set case}) is bounded as
\begin{align}
&\sum_{b=1}^{b_T}
\mathbbm1\left\{
\langle\hat{\bm a}^*(T_{b})\rangle\neq[L],\
\mathcal O(T_{b})=\emptyset,\
\mathcal U(T_{b})\neq\emptyset
\right\}\nonumber\\
\le&
\sum_{b=1}^{b_T}
\sum_{j=1}^L
\mathbbm1\left\{
\hat\theta_j(T_{b})
\le\theta_j-\frac{\Delta_{L+1}}{2},\
\mathcal O(T_{b})=\emptyset
\right\}\nonumber\\
\le&
\sum_{b=1}^{b_T}
\sum_{j=1}^L
\left(
\mathbbm1\left(
\left\{
\hat\theta_j(T_{b})
\le\theta_j-\frac{\Delta_{L+1}}{2},\
j\in\langle\bm a(T_{b})\rangle
\right\}
\right.\right.\nonumber\\
&\left.\left.\qquad\cup
\left\{
\hat\theta_j(T_{b})
\le
\hat\theta_{i_{T_{b}}^{-1}(L)}(T_{b})
<
\theta_j-\frac{\Delta_{L+1}}{2},
\right.\right.\right.\nonumber\\
&\left.\left.\left.\qquad\qquad
n_j(T_{b}-1)
d\left(
\hat\theta_j(T_{b}),
\hat\theta_{i_{T_{b}}^{-1}(L)}(T_{b})
\right)
<(1+\alpha)\log(T_{b}-1)
\right\}
\right)
\right.\nonumber\\
&\left.\qquad+
\mathbbm1\left\{
\hat\theta_j(T_{b})
\le
\hat\theta_{i_{T_{b}}^{-1}(L)}(T_{b})
<
\theta_j-\frac{\Delta_{L+1}}{2},
\right.\right.\nonumber\\
&\left.\left.\qquad\qquad
n_j(T_{b}-1)
d\left(
\hat\theta_j(T_{b}),
\hat\theta_{i_{T_{b}}^{-1}(L)}(T_{b})
\right)
\ge(1+\alpha)\log(T_{b}-1)
\right\}
\right),\label{ineq:U-case}
\end{align}
where, the inequality
\[
\hat\theta_{i_{T_{b}}^{-1}(L)}(T_{b})
<
\theta_j-\frac{\Delta_{L+1}}{2}
\]
is derived from
\(\mathcal U(T_{b})\neq\emptyset\) and
\(i_{T_{b}}(j)>L\). The first term of the right-hand side of Ineq.~(\ref{ineq:U-case}) is bounded as 
\begin{align*}
&\sum_{b=1}^{b_T}
\sum_{j=1}^L
\mathbbm1\left(
\left\{
\hat\theta_j(T_{b})
\le\theta_j-\frac{\Delta_{L+1}}{2},\
j\in\langle\bm a(T_{b})\rangle
\right\}
\right.\\
&\left.\qquad\cup
\left\{
\hat\theta_j(T_{b})
\le
\hat\theta_{i_{T_{b}}^{-1}(L)}(T_{b})
<
\theta_j-\frac{\Delta_{L+1}}{2},
\right.\right.\\
&\left.\left.\qquad\qquad
n_j(T_{b}-1)
d\left(
\hat\theta_j(T_{b}),
\hat\theta_{i_{T_{b}}^{-1}(L)}(T_{b})
\right)
<(1+\alpha)\log(T_{b}-1)
\right\}
\right)\\
\le&
\sum_{b=1}^{b_T}
\sum_{j=1}^L
\mathbbm1\left\{
\begin{gathered}
\hat\theta_j(T_{b})
\le\theta_j-\frac{\Delta_{L+1}}{2},\\
\exists t\in\{T_{b},\dots,T_{b+1}-1\}
\text{ such that }j\in\langle\bm a(t)\rangle
\end{gathered}
\right\}\\
=&
\sum_{b=1}^{b_T}
\sum_{j=1}^L
\sum_{s=1}^T
\mathbbm1\left\{
\begin{gathered}
\hat\theta_j(T_{b})
\le\theta_j-\frac{\Delta_{L+1}}{2},\
n_j(T_{b}-1)=s,\\
\exists t\in\{T_{b},\dots,T_{b+1}-1\}
\text{ such that }j\in\langle\bm a(t)\rangle
\end{gathered}
\right\}\\
=&
\sum_{j=1}^L
\sum_{s=1}^T
\sum_{b=1}^{b_T}
\mathbbm1\left\{
\begin{gathered}
\hat\theta_j(T_{b})
\le\theta_j-\frac{\Delta_{L+1}}{2},\
n_j(T_{b}-1)=s,\\
\exists t\in\{T_{b},\dots,T_{b+1}-1\}
\text{ such that }j\in\langle\bm a(t)\rangle
\end{gathered}
\right\}\\
\le&
\sum_{j=1}^L
\sum_{s=1}^T
Z_{j,s}
\mathbbm1\left\{
\hat\theta_{j,s}
\le\theta_j-\frac{\Delta_{L+1}}{2}
\right\}.
\end{align*}
Therefore,
{\small
\begin{align*}
&\mathbb E\left[
\sum_{b=1}^{b_T}
\sum_{j=1}^L
\mathbbm1\left(
\left\{
\hat\theta_j(T_{b})
\le\theta_j-\frac{\Delta_{L+1}}{2},\
j\in\langle\bm a(T_{b})\rangle
\right\}
\right.\right.\\
&\left.\left.\qquad\cup
\left\{
\hat\theta_j(T_{b})
\le
\hat\theta_{i_{T_{b}}^{-1}(L)}(T_{b})
<
\theta_j-\frac{\Delta_{L+1}}{2},
\right.\right.\right.\\
&\left.\left.\left.\qquad\qquad
n_j(T_{b}-1)
d\left(
\hat\theta_j(T_{b}),
\hat\theta_{i_{T_{b}}^{-1}(L)}(T_{b})
\right)
<(1+\alpha)\log(T_{b}-1)
\right\}
\right)
\right]\\
\le&
\sum_{j=1}^L
\sum_{s=1}^T
\mathbb E[Z_{j,s}]
\mathbb P\left\{
\hat\theta_{j,s}
\le\theta_j-\frac{\Delta_{L+1}}{2}
\right\}\\
\le&
\sum_{j=1}^L
\sum_{s=1}^T
\frac{1}{P_L}
e^{-s\frac{\Delta_{L+1}^2}{2}}
\tag{by Hoeffding's inequality}\\
\le&
\frac{L}{P_L}
\frac{1}{e^{\frac{\Delta_{L+1}^2}{2}}-1}
\le
\frac{L}{P_L}
\frac{2}{\Delta_{L+1}^2}.
\end{align*}
}

Let \(f_\theta(x)=\KL(x,\theta)\) for \(0\le x\le\theta\).
Then, the second term of the right-hand side of Ineq~(\ref{ineq:U-case}) is bounded as
\begin{align*}
&\sum_{b=1}^{b_T}
\sum_{j=1}^L
\mathbbm1\left\{
\hat\theta_j(T_{b})
\le
\hat\theta_{i_{T_{b}}^{-1}(L)}(T_{b})
<
\theta_j-\frac{\Delta_{L+1}}{2},
\right.\\
&\left.\qquad
n_j(T_{b}-1)
\KL\left(
\hat\theta_j(T_{b}),
\hat\theta_{i_{T_{b}}^{-1}(L)}(T_{b})
\right)
\ge(1+\alpha)\log(T_{b}-1)
\right\}\\
\le&
\sum_{b=1}^{b_T}
\sum_{j=1}^L
\mathbbm1\left\{
\hat\theta_j(T_{b})
<
\theta_j-\frac{\Delta_{L+1}}{2},\
\right.\\
&\left.\qquad
n_j(T_{b}-1)
\KL\left(
\hat\theta_j(T_{b}),
\theta_j-\frac{\Delta_{L+1}}{2}
\right)
\ge(1+\alpha)\log(T_{b}-1)
\right\}\\
\le&
\sum_{b=1}^{b_T}
\sum_{j=1}^L
\mathbbm1\left\{
\KL\left(
0,\theta_j-\frac{\Delta_{L+1}}{2}
\right)
\le
\frac{(1+\alpha)\log(T_{b}-1)}{n_j(T_{b}-1)},
\right.\\
&\left.\qquad
\hat\theta_j(T_{b})
<
f_{\theta_j-\frac{\Delta_{L+1}}{2}}^{-1}
\left(
\frac{(1+\alpha)\log(T_{b}-1)}{n_j(T_{b}-1)}
\right)
\right\}.
\end{align*}
Thus,
\begin{align*}
&\mathbb E\left[
\sum_{b=1}^{b_T}
\sum_{j=1}^L
\mathbbm1\left\{
\hat\theta_j(T_{b})
\le
\hat\theta_{i_{T_{b}}^{-1}(L)}(T_{b})
<
\theta_j-\frac{\Delta_{L+1}}{2},
\right.\right.\\
&\left.\left.\qquad
n_j(T_{b}-1)
\KL\left(
\hat\theta_j(T_{b}),
\hat\theta_{i_{T_{b}}^{-1}(L)}(T_{b})
\right)
\ge(1+\alpha)\log(T_{b}-1)
\right\}
\right]\\
\le&
\sum_{b=1}^{b_T}
\sum_{j=1}^L
e^{-n_j(T_{b}-1)
\KL\left(
f_{\theta_j-\frac{\Delta_{L+1}}{2}}^{-1}
\left(
\frac{(1+\alpha)\log(T_{b}-1)}{n_j(T_{b}-1)}
\right),
\theta_j
\right)}\\
\le&
\sum_{j=1}^L
\sum_{b=1}^{b_T}
e^{-n_j(T_{b-1})
\KL\left(
f_{\theta_j-\frac{\Delta_{L+1}}{2}}^{-1}
\left(
\frac{(1+\alpha)\log(T_{b}-1)}{n_j(T_{b}-1)}
\right),
\theta_j-\frac{\Delta_{L+1}}{2}
\right)}\\
=&
\sum_{j=1}^L
\sum_{b=1}^{b_T}
e^{-n_j(T_{b-1})
\cdot
\frac{(1+\alpha)\log(T_{b}-1)}{n_j(T_{b}-1)}}\\
=&
\sum_{j=1}^L
\sum_{b=1}^{b_T}
\frac{1}{(T_{b}-1)^{1+\alpha}}
\le
L\zeta(1+\alpha),
\end{align*}

where \(\zeta\) is the Riemann zeta function.
Therefore, theorem's inequality holds.
%\qed
\end{proof}

\begin{lemma}
\label{lem:topLrankdiff}
The following inequality holds:
\begin{align*}
&\mathbb E\left[
\sum_{b=1}^{b_T}
\sum_{t=T_{b}}^{T_{b+1}-1}
\mathbbm1\left\{
\langle\hat{\bm a}^*(T_{b})\rangle=[L],\
\hat{\bm a}^*(T_{b})\neq\bm a^*
\right\}
\right]\le
\frac{4K}{P_L}
\sum_{j=1}^{L-1}
\frac{1}{\Delta_{j,j+1}^2}.
\end{align*}
\end{lemma}

\begin{proof}
\begin{align*}
&\sum_{b=1}^{b_T}
\sum_{t=T_{b}}^{T_{b+1}-1}
\mathbbm1\left\{
\langle\hat{\bm a}^*(T_{b})\rangle=[L],\
\hat{\bm a}^*(T_{b})\neq\bm a^*
\right\}\\
\le&
K\sum_{b=1}^{b_T}
\mathbbm1\left\{
\langle\hat{\bm a}^*(T_{b})\rangle=[L],\
\hat{\bm a}^*(T_{b})\neq\bm a^*
\right\}\\
\le&
K\sum_{b=1}^{b_T}
\sum_{j=1}^{L-1}
\mathbbm1\left\{
\hat\theta_j(T_{b})
\le
\theta_j-\frac{\Delta_{j,j+1}}{2},\
j\in\langle\bm a(T_{b})\rangle
\right\}\\
&+
K\sum_{b=1}^{b_T}
\sum_{j=2}^L
\mathbbm1\left\{
\hat\theta_j(T_{b})
\ge
\theta_j+\frac{\Delta_{j-1,j}}{2},\
j\in\langle\bm a(T_{b})\rangle
\right\}.
\end{align*}

\begin{align*}
&\mathbb E\left[
\sum_{b=1}^{b_T}
\sum_{j=1}^{L-1}
\mathbbm1\left\{
\hat\theta_j(T_{b})
\le
\theta_j-\frac{\Delta_{j,j+1}}{2},\
j\in\langle\bm a(T_{b})\rangle
\right\}
\right]\\
\le&
\sum_{j=1}^{L-1}
\sum_{s=1}^T
\frac{1}{P_L}
e^{-s\frac{\Delta_{j,j+1}^2}{2}}
\tag{by Hoeffding's inequality}\\
\le&
\sum_{j=1}^{L-1}
\frac{2}{P_L\Delta_{j,j+1}^2}.
\end{align*}

\begin{align*}
&\mathbb E\left[
\sum_{b=1}^{b_T}
\sum_{j=2}^L
\mathbbm1\left\{
\hat\theta_j(T_{b})
\ge
\theta_j+\frac{\Delta_{j-1,j}}{2},\
j\in\langle\bm a(T_{b})\rangle
\right\}
\right]\\
\le&
\sum_{j=2}^L
\sum_{s=1}^T
\frac{1}{P_L}
e^{-s\frac{\Delta_{j-1,j}^2}{2}}
\tag{by Hoeffding's inequality}\\
\le&
\sum_{j=2}^L
\frac{2}{P_L\Delta_{j-1,j}^2}
=
\sum_{j=1}^{L-1}
\frac{2}{P_L\Delta_{j,j+1}^2}.
\end{align*}
Thus, theorem's inequality holds.
%\qed
\end{proof}

For \(L<j\le K\), define \(\delta_j(\epsilon)>0\) as the value that
satisfies
\[
\KL\left(
\theta_j+\delta_j(\epsilon),
\theta_L-\delta_j(\epsilon)
\right)
=
\frac{\KL(\theta_j,\theta_L)}{1+\epsilon}.
\]

\begin{lemma}
\label{lem:suboptdrawnum}
For any \(\epsilon>0\), the following inequality holds:
\begin{align*}
&\mathbb E\left[
\sum_{b=1}^{b_T}
\sum_{t=T_{b}}^{T_{b+1}-1}
\mathbbm1\left\{
\hat{\bm a}^*(T_{b})=\bm a^*
\right\}
\Delta_{\bm a(t)}
\right]\\
&\quad\le
\sum_{j=L+1}^K
\left(
\frac{(1+\alpha)(1+\epsilon)\log T}
{\KL(\theta_j,\theta_L)}
+\frac{1}{(\delta_j(\epsilon))^2}
\right)
\Delta_j.
\end{align*}
\end{lemma}

\begin{proof}
We have 
\begin{align}
&\sum_{b=1}^{b_T}
\sum_{t=T_{b}}^{T_{b+1}-1}
\mathbbm1\left\{
\hat{\bm a}^*(T_{b})=\bm a^*
\right\}
\Delta_{\bm a(t)}\nonumber\\
\le&
\sum_{b=1}^{b_T}
\sum_{j=L+1}^K
\mathbbm1\left\{
n_j(T_{b}-1)
\KL\left(
\hat\theta_j(T_{b}),
\hat\theta_L(T_{b})
\right)
<(1+\alpha)\log(T_{b}-1)
\right\}
P_L\Delta_j\nonumber\\
=&
\sum_{j=L+1}^K
P_L\Delta_j
\sum_{b=1}^{b_T}
\mathbbm1\left\{
n_j(T_{b}-1)
\KL\left(
\hat\theta_j(T_{b}),
\hat\theta_L(T_{b})
\right)
<(1+\alpha)\log(T_{b}-1)
\right\}\nonumber\\
\le&
\sum_{j=L+1}^K
P_L\Delta_j
\sum_{b=1}^{b_T}
\left(
\mathbbm1\left\{
\hat\theta_j(T_{b})
<
\theta_j+\delta_j(\epsilon),\
\hat\theta_L(T_{b})
>
\theta_L-\delta_j(\epsilon),
\right.\right.\nonumber\\
&\left.\left.\qquad
n_j(T_{b}-1)
\KL\left(
\hat\theta_j(T_{b}),
\hat\theta_L(T_{b})
\right)
<(1+\alpha)\log(T_{b}-1)
\right\}
\right.\nonumber\\
&\left.\qquad+
\mathbbm1\left\{
\hat\theta_j(T_{b})
\ge\theta_j+\delta_j(\epsilon)
\right\}
+\mathbbm1\left\{
\hat\theta_L(T_{b})
\le\theta_L-\delta_j(\epsilon)
\right\}
\right)\nonumber\\
\le&
\sum_{j=L+1}^K
P_L\Delta_j
\sum_{b=1}^{b_T}
\left(
\mathbbm1\left\{
n_j(T_{b}-1)
\KL\left(
\theta_j+\delta_j(\epsilon),
\theta_L-\delta_j(\epsilon)
\right)
<(1+\alpha)\log(T_{b}-1)
\right\}
\right.\nonumber\\
&\left.\qquad+
\mathbbm1\left\{
\hat\theta_j(T_{b})
\ge\theta_j+\delta_j(\epsilon)
\right\}
+\mathbbm1\left\{
\hat\theta_L(T_{b})
\le\theta_L-\delta_j(\epsilon)
\right\}
\right).\label{ineq:correct-set-case}
\end{align}

The summation of the first term over $b=1,\dots,b_T$ is upper bounded as 
\begin{align*}
&\sum_{b=1}^{b_T}
\mathbbm1\left\{
n_j(T_{b}-1)
\KL\left(
\theta_j+\delta_j(\epsilon),
\theta_L-\delta_j(\epsilon)
\right)
<(1+\alpha)\log(T_{b}-1)
\right\}\\
=&
\sum_{b=1}^{b_T}
\mathbbm1\left\{
n_j(T_{b}-1)
\frac{\KL(\theta_j,\theta_L)}{1+\epsilon}
<(1+\alpha)\log(T_{b}-1)
\right\}\\
\le&
\sum_{b=1}^{b_T}
\sum_{s=1}^T
\mathbbm1\left\{
n_j(T_{b}-1)
\frac{\KL(\theta_j,\theta_L)}{1+\epsilon}
<(1+\alpha)\log T,\
n_j(T_{b}-1)=s
\right\}\\
\le&
\sum_{s=1}^T
Z_{j,s}
\mathbbm1\left\{
s\frac{\KL(\theta_j,\theta_L)}{1+\epsilon}
<(1+\alpha)\log T
\right\}.
\end{align*}
Thus, the expectation of the first term sum over $b=1,\dots,b_T$ is bounded as
\begin{align*}
&\mathbb E\left[
\sum_{b=1}^{b_T}
\mathbbm1\left\{
n_j(T_{b}-1)
\KL\left(
\theta_j+\delta_j(\epsilon),
\theta_L-\delta_j(\epsilon)
\right)
<(1+\alpha)\log(T_{b}-1)
\right\}
\right]\\
\le&
\sum_{s=1}^{
\left\lfloor
\frac{(1+\alpha)(1+\epsilon)\log T}
{\KL(\theta_j,\theta_L)}
\right\rfloor
}
\mathbb E[Z_{j,s}]\\
=&
\frac{(1+\alpha)(1+\epsilon)\log T}
{P_L\KL(\theta_j,\theta_L)}.
\end{align*}
The summation of the second term over $b=1,\dots,b_T$ in the right-hand side of Ineq.~(\ref{ineq:correct-set-case}) is upper bounded as 
\begin{align*}
&\sum_{b=1}^{b_T}
\mathbbm1\left\{
\hat\theta_j(T_{b})
\ge\theta_j+\delta_j(\epsilon)
\right\}\\
=&
\sum_{b=1}^{b_T}
\sum_{s=1}^T
\mathbbm1\left\{
\hat\theta_j(T_{b})
\ge\theta_j+\delta_j(\epsilon),\
n_j(T_{b}-1)=s
\right\}\\
\le&
\sum_{s=1}^T
Z_{j,s}
\mathbbm1\left\{
\hat\theta_{j,s}
\ge\theta_j+\delta_j(\epsilon)
\right\}.
\end{align*}
Thus, its expectation is upper bounded as
\begin{align*}
&\mathbb E\left[
\sum_{b=1}^{b_T}
\mathbbm1\left\{
\hat\theta_j(T_{b})
\ge\theta_j+\delta_j(\epsilon)
\right\}
\right]\\
\le&
\sum_{s=1}^T
\mathbb E[Z_{j,s}]
\mathbb P\left\{
\hat\theta_{j,s}
\ge\theta_j+\delta_j(\epsilon)
\right\}\\
\le&
\sum_{s=1}^T
\frac{1}{P_L}
e^{-2s(\delta_j(\epsilon))^2}
\tag{by Hoeffding's inequality}\\
\le&
\frac{1}{2P_L(\delta_j(\epsilon))^2}.
\end{align*}
The summation of the third term over $b=1,\dots,b_T$ in the right-hand side of Ineq.~(\ref{ineq:correct-set-case}) is upper bounded as
\begin{align*}
&\sum_{b=1}^{b_T}
\mathbbm1\left\{
\hat\theta_L(T_{b})
\le\theta_L-\delta_j(\epsilon)
\right\}\\
=&
\sum_{b=1}^{b_T}
\sum_{s=1}^T
\mathbbm1\left\{
\hat\theta_L(T_{b})
\le\theta_L-\delta_j(\epsilon),\
n_L(T_{b}-1)=s
\right\}\\
\le&
\sum_{s=1}^T
Z_{L,s}
\mathbbm1\left\{
\hat\theta_{L,s}
\le\theta_j-\delta_j(\epsilon)
\right\}.
\end{align*}
Hence, its expectation is upper bounded as
\begin{align*}
&\mathbb E\left[
\sum_{b=1}^{b_T}
\mathbbm1\left\{
\hat\theta_L(T_{b})
\le\theta_L+\delta_j(\epsilon)
\right\}
\right]\\
\le&
\sum_{s=1}^T
\mathbb E[Z_{L,s}]
\mathbb P\left\{
\hat\theta_{L,s}
\le\theta_L+\delta_j(\epsilon)
\right\}\\
\le&
\sum_{s=1}^T
\frac{1}{P_L}
e^{-2s(\delta_j(\epsilon))^2}
\tag{by Hoeffding's inequality}\\
\le&
\frac{1}{2P_L(\delta_j(\epsilon))^2}.
\end{align*}
Thus, lemma's inequality holds.
%\qed
\end{proof}

\begin{proof}[Proof of Theorem~\ref{thm:upper-bound}]
We have
\begin{align*}
&\mathrm{Regret}_{\text{OD-DMED-L}}(T)\\
=&\sum_{t=1}^T\Delta_{\bm a(t)}\\
=&\sum_{t=1}^{T_1-1}\Delta_{\bm a(t)}
+\sum_{b=1}^{b_T}
\sum_{t=T_{b}}^{T_{b+1}-1}
\Delta_{\bm a(t)}\\
\le&
\sum_{t=1}^{T_1-1}\Delta_{\bm a(t)}
+\sum_{b=1}^{b_T}
\sum_{t=T_{b}}^{T_{b+1}-1}
\mathbbm1\left\{
\langle\hat{\bm a}^*(T_{b})\rangle
\neq[L]
\right\}
\Delta_{\bm a(t)}\\
&+\sum_{b=1}^{b_T}
\sum_{t=T_{b}}^{T_{b+1}-1}
\mathbbm1\left\{
\langle\hat{\bm a}^*(T_{b})\rangle
=[L],\
\hat{\bm a}^*(T_{b})\neq\bm a^*
\right\}
\Delta_{\bm a(t)}\\
&+\sum_{b=1}^{b_T}
\sum_{t=T_{b}}^{T_{b+1}-1}
\mathbbm1\left\{
\hat{\bm a}^*(T_{b})=\bm a^*
\right\}
\Delta_{\bm a(t)}.
\end{align*}

Thus, for any \(\epsilon>0\),
\begin{align*}
&\mathbb E\left[
\mathrm{Regret}_{\text{OD-DMED-L}}(T)
\right]\\
\le&
\mathbb E\left[
\sum_{t=1}^{T_1-1}\Delta_{\bm a(t)}
\right]\\
&+\mathbb E\left[
\sum_{b=1}^{b_T}
\sum_{t=T_{b}}^{T_{b+1}-1}
\mathbbm1\left\{
\langle\hat{\bm a}^*(T_{b})\rangle
\neq[L]
\right\}
\Delta_{\bm a(t)}
\right]\\
&+\mathbb E\left[
\sum_{b=1}^{b_T}
\sum_{t=T_{b}}^{T_{b+1}-1}
\mathbbm1\left\{
\langle\hat{\bm a}^*(T_{b})\rangle
=[L],\
\hat{\bm a}^*(T_{b})\neq\bm a^*
\right\}
\Delta_{\bm a(t)}
\right]\\
&+\mathbb E\left[
\sum_{b=1}^{b_T}
\sum_{t=T_{b}}^{T_{b+1}-1}
\mathbbm1\left\{
\hat{\bm a}^*(T_{b})=\bm a^*
\right\}
\Delta_{\bm a(t)}
\right]\\
\le&
KL
+\frac{KL}{P_L}
\sum_{j=L+1}^K\frac{2}{\Delta_j^2}
+\frac{KL^2}{P_L}\frac{2}{\Delta_{L+1}^2}
+KL^2\zeta(1+\alpha)
+\frac{4KL}{P_L}
\sum_{j=1}^{L-1}\frac{1}{\Delta_{j,j+1}^2}\\
&+\sum_{j=L+1}^K
\left(
\frac{(1+\alpha)(1+\epsilon)\log T}
{\KL(\theta_j,\theta_L)}
+\frac{1}{(\delta_j(\epsilon))^2}
\right)
\Delta_j.
\end{align*}
The last inequality follows from
\(\Delta_{\bm a}\le L\) for \(\bm a\in\Pi([K],L)\) and
Lemmas~\ref{lem:topLdiff}, \ref{lem:topLrankdiff}, and
\ref{lem:suboptdrawnum}.

Thus, we obtain
\[
\limsup_{T\to\infty}
\frac{\E[\Reg_{\text{OD-DMED-L}}(T)]}{\log T}
\le
(1+\alpha)
\sum_{j=L+1}^K
\frac{\Delta_j}
{\KL(\theta_j,\theta_L)}
\]
by \(\epsilon\rightarrow+0\), because
\(\delta_j(\epsilon)\rightarrow0\) as
\(\epsilon\rightarrow+0\).
%\qed
\end{proof}

\section{Comparison Algorithms}\label{app:seccomp}

\begin{description}
    \item[PBM-UCB \cite{Lagree2016PBM}:] This algorithm presents the $L$ items with the largest UCB indices defined as
\begin{align}
     \mathrm{UCB}_i^{\mathrm{PBM}}(t)=\frac{s_i(t)}{\tilde{N}_i(t)}+\sqrt{\frac{N_i(t)}{\tilde{N}_i(t)}}\sqrt{\frac{\delta(t)}{2\tilde{N}_i(t)}},\label{eq:PBMUCB}
\end{align}
where
\begin{align*}
    N_i(t)=&\sum_{s=1}^t\sum_{j=1}^L\mathbbm{1}[a_j(s)=i] \text{ and}\\
    \tilde{N}_i(t)=&\sum_{s=1}^t\sum_{j=1}^LP_j\mathbbm{1}[a_j(s)=i].\\  
\end{align*}
\item[PBM-TS \cite{Lagree2016PBM}:] The ranking of this algorithm is constructed by selecting the items $i$ with the largest sample $\tilde{\theta}_i$ drawn from the distribution
\[
\pi_t(\theta_i)\propto \prod_{j=1}^L\theta_i^{\alpha_{i,j}(t)}(1-P_j\theta_i)^{\beta_{i,j}(t)},
\]
where 
\begin{align*}
   \alpha_{i,j}(t)=&\sum_{s=1}^t\mathbbm{1}[a_j(s)=i]C_i(s) \text{ and}\\
   \beta_{i,j}(t)=&\sum_{s=1}^t\mathbbm{1}[a_j(s)=i]-\alpha_{i,j}(t).
\end{align*}
Sampling from this distribution can be realized by rejection sampling with distribution
$\mathrm{Beta}(\alpha_{i,m(t)}(t),\beta_{i,m(t)}(t))/P_{m(t)}$, where $m(t)=\argmax_{j\in[L]}(\alpha_{i,j}(t)+\beta_{i,j}(t))$.

\item[LastClick-UCB:] This algorithm is the multiple-click version of CascadeUCB1 \cite{Kveton2015Cascading}.
At round $t$, it selects $\bm{a}\in\Pi([K],L)$ that minimizes $\prod_{j=1}^L(1-\mathrm{UCM}_{a_j}^{\mathrm{Cascade}}(t))$, where
\[
\mathrm{UCM}_{i}^{\mathrm{Cascade}}(t)=\frac{s_i(t)}{n'_i(t)}+\sqrt{\frac{\alpha\log t}{n'_i(t)}}.
\]
Note that $\alpha$ is a constant of \emph{explore scale}.
Here, $n'_i(t)$ is defined as
\[
n'_i(t)=\sum_{s=1}^t\sum_{j=1}^{V'(t)}\mathbbm{1}[a_j(s)=i]
\]
using $V'(t)$, which is the last click position if some items in $\langle\bm{a}(t)\rangle$ are clicked, and $L$ otherwise.
\end{description}

\end{document}